\documentclass[letterpaper]{article} 
\usepackage{aaai25}  
\usepackage{times}  
\usepackage{helvet}  
\usepackage{courier}  
\usepackage[hyphens]{url}  
\usepackage{graphicx} 
\usepackage{natbib}  
\usepackage{caption} 
\usepackage{booktabs} 
\usepackage{multirow} 
\usepackage{amsmath}
\usepackage{amsthm}
\usepackage{amssymb}
\usepackage{colortbl}
\usepackage{color,xcolor}
\newtheorem{theorem}{Theorem}

\newtheorem{definition}{Definition}

\newtheorem{proposition}{Proposition}
\newtheorem{assumption}{Assumption}
\newtheorem{Def}{Definition}
\newtheorem{Assum}{Assumption}
\newtheorem{Theom}{Theorem}

\newtheorem{Prop}{Proposition}
\newtheorem{Fact}{Fact}
\newtheorem{Hypo}{Induction Hypothesis}

\newcommand{\rowbg}{gray!15}

\usepackage{algorithm}
\usepackage{algorithmic}

\usepackage{newfloat}
\usepackage{listings}
\DeclareCaptionStyle{ruled}{labelfont=normalfont,labelsep=colon,strut=off} 
\floatstyle{ruled}
\newfloat{listing}{tb}{lst}{}
\floatname{listing}{Listing}
\title{Mining In-distribution Attributes in Outliers for Out-of-distribution Detection}
\author{
    Yutian Lei, Luping Ji\footnote{Corresponding author.}, Pei Liu
}
\affiliations{

    School of Computer Science and Engineering, University of Electronic Science and Technology of China, Chengdu, China\\
    ytlei823@gmail.com, jiluping@uestc.edu.cn, yuukilp@std.uestc.edu.cn
}

\usepackage{bibentry}

\begin{document}

\maketitle

\begin{abstract}
Out-of-distribution (OOD) detection is indispensable for deploying reliable machine learning systems in real-world scenarios. Recent works, using auxiliary outliers in training, have shown good potential. However, they seldom concern the intrinsic correlations between in-distribution (ID) and OOD data. In this work, we discover an obvious correlation that OOD data usually possesses significant ID attributes. These attributes should be factored into the training process, rather than blindly suppressed as in previous approaches. Based on this insight, we propose a structured multi-view-based out-of-distribution detection learning (MVOL) framework, which facilitates rational handling of the intrinsic in-distribution attributes in outliers. We provide theoretical insights on the effectiveness of MVOL for OOD detection. Extensive experiments demonstrate the superiority of our framework to others. MVOL effectively utilizes both auxiliary OOD datasets and even wild datasets with noisy in-distribution data. Code is available at  \textcolor{gray}{\url{https://github.com/UESTC-nnLab/MVOL}}.
\end{abstract}


%
\section{Introduction}

Modern deep neural networks could produce overconfident and unreliable predictions when test inputs are out-of-distribution (OOD) \cite{Nguyen_2015_CVPR}. This presents a crucial challenge, especially for deploying deep learning models in the real world. To tackle this challenge, recent studies have explored outlier exposure (OE)-based training strategies \cite{hendrycks2018oe}. They utilize an auxiliary outlier dataset in training to suppress models' response to outliers, \textit{i.e.}, out-of-distribution data, thus detecting those inputs far away from in-distribution (ID) data \cite{chen2021atom,ming2022poem,zhu2023DivOE}. These strategies have achieved considerable success and often perform better than those without auxiliary data \cite{VOS,ASH}.

Despite the promising results, existing OE-based methods pay limited attention to what underlying correlations exist between ID and OOD data, still suffering from the irrational use of auxiliary outliers. As shown in Figure \ref{Motivation}(a), for models trained solely on ID data, we find they tend to present higher logits (\textit{i.e.}, pre-softmax output) on specific known categories than others when their inputs are out-of-distribution. This finding implies that outliers could contain ID attributes, although they are often believed to be semantically distinct from in-distribution \cite{bai2023feed2birds}. However, previous approaches, \textit{e.g.}, outlier exposure \cite{hendrycks2018oe}, and energy-regularized learning \cite{liu2020energy}, usually ignore those intrinsic ID attributes. Specifically, they treat these attributes as entirely random noise and attempt to suppress models' responses to them blindly. Such behavior indicates unreasonable use of auxiliary outliers and may degrade the detection performance. The above discrepancy naturally motivates the following question: \textit{how can we rationally handle the intrinsic ID attributes in auxiliary outlier data for OOD detection?} 

To address this question, we propose a structured multi-view-based out-of-distribution learning framework (\textbf{MVOL}) in Figure \ref{Motivation} (d) --- tackling OOD detection via mining ID attributes in outliers.
\textit{In data level}, this framework involves an extended multi-view data model. It establishes ID and OOD data in a unified perspective, revealing intrinsic in-distribution attributes in outliers.
\textit{In model level}, MVOL employs maximum logit (MaxLogit) as the OOD score with new insight, which naturally adapts to our data model. 
To calibrate unexpected high MaxLogit produced for outlier input, MVOL involves a multi-view-based learning objective to rationally utilize ID attributes in auxiliary outliers.

\begin{figure*}[tp]
\centering
\includegraphics[scale=0.42]{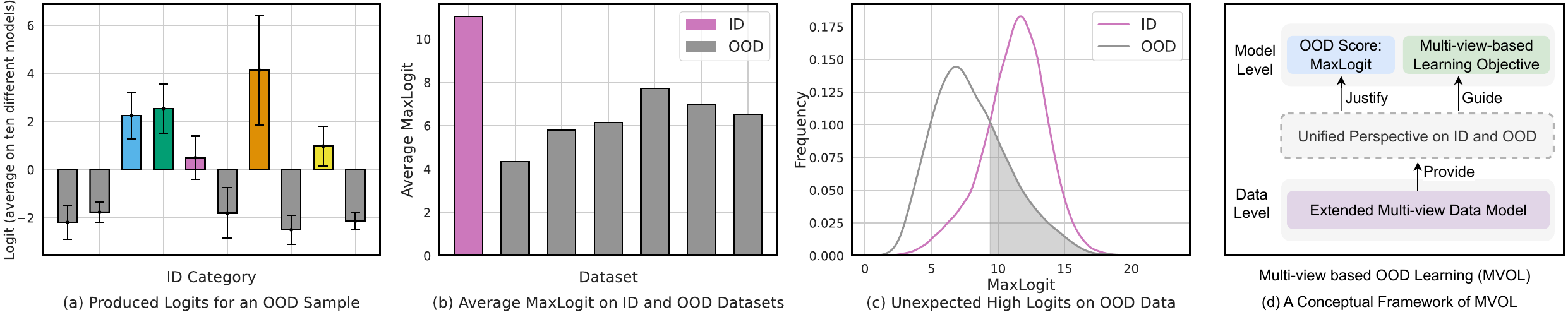}
\caption{
(a)-(c) Motivation and (d) overview of MVOL. (a) OOD data could have ID attributes, (b) outliers mainly have minor ID features, and (c) the model trained solely on ID data can produce over-activated logits for certain outliers. The gray-shaded area highlights those outliers with high probability. 
Refer to Appendix E for more experimental details.
}
\label{Motivation}
\end{figure*}
\noindent\textbf{Extended Multi-view Data Model (MVDM).}~The original MVDM \cite{allen-zhu2023towards} posits that in a natural dataset, a data point comprises main features, minor features, and noise \footnote{
In a vision dataset featuring car and cat categories, a car image typically comprises headlights, wheels, or windows, termed main features. Additionally, some cars might have headlights resembling cat eyes, which can also be recognized as a small fraction of ``cat features", termed minor features.
}. To formally study the correlations between ID and OOD data, we extend the original MVDM to the context of out-of-distribution detection by assuming that \textit{outliers mainly consist of minor ID features and noise}. This new assumption originates from our two empirical findings. \textbf{(1)} Outliers contain ID attributes, as indicated by Figure \ref{Motivation}(a) and analyzed above. \textbf{(2)} These ID attributes generally have lower magnitudes than those in ID data. As shown in Figure \ref{Motivation} (b), by comparing averages of maximum logit on ID and six OOD datasets, we observe that ID data has a significantly larger value. This observation suggests that ID attributes in outliers correspond to minor ID features. This is because compared to main ID features, minor ones typically carry a smaller weight, leading a well-trained model to produce smaller logit responses, as shown in MVDM.

\noindent\textbf{New Insight on MaxLogit as an OOD Score.}~Beyond empirical evidence in previous works \cite{hendrycks2022maxlogit,vaze2022openset}, our extended MVDM justifies that MaxLogit serves as an interpretable and effective OOD score. (1) \textbf{Interpretable}: MaxLogit measures ID attributes contained in test inputs. Outliers with minor ID features tend to have a low score.
(2) \textbf{Effective}: When a well-calibrated model learns all ID features, its error bound is near zero theoretically.
Nevertheless, for a model trained solely on ID data, it could still be uncalibrated and produce unexpected high MaxLogit for outlier inputs, as shown in Figure \ref{Motivation}(c), which degrades the detection performance \cite{wei2022logitNorm}.

\noindent\textbf{Multi-view-based Learning Objective.}~
We propose a learning objective for OOD detection --- explicitly utilizing the intrinsic minor ID features in auxiliary outliers to calibrate unexpected high logits.
This is in contrast with confidence loss \cite{Kimin2018training} in OE, which forces models to respond \textit{uniformly} to all ID categories for outliers. It causes models to underemphasize minor ID features yet overemphasize noise, thus violating the heuristics that minor ID features deserve larger logit responses than noise. Accordingly, our learning process encourages models to more accurately respond to the minor ID features and noise in outliers. 

\noindent\textbf{Contributions:}
\textbf{(1)} To our knowledge, 
we are among the first works that explicitly mine in-distribution attributes in outliers to tackle out-of-distribution detection. 
\textbf{(2)} We propose a multi-view-based out-of-distribution learning framework (\textbf{MVOL}) to handle the intrinsic ID attributes in outliers rationally. 
\textbf{(3)} We show that MVOL performs overall strong performance while training with OOD datasets and even wild datasets with ID data as noise.

\section{Related Work}
\textbf{Out-of-distribution Detection.}
\cite{hendrycks2017MSP} presents a baseline for OOD detection using a pre-trained model's maximum softmax probability. Subsequent works develop different scoring functions, such as ODIN \cite{shiyu2018odin}, Mahalanobis \cite{lee2018MAHA}, Energy \cite{liu2020energy}, MaxLogit \cite{hendrycks2022maxlogit,vaze2022openset}. Some methods synthesize outliers to regularize models. Synthesized outliers are represented in feature space \cite{VOS,tao2023nonparametric} or pixel space \cite{ATOL,dreamOOD}. 

Another promising line of work uses auxiliary outliers to regularize models. 
Outlier exposure \cite{hendrycks2018oe} using confidence loss \cite{Kimin2018training} encourages models to output a uniform distribution of softmax probability on outliers. Energy-based regularization (Energy) \cite{liu2020energy} trains models to widen the energy gap by increasing the sum of logits for ID samples above a lower bound and reducing those of outliers below an upper bound. Both OE and Energy fail to consider ID and OOD correlations explicitly. Some other works highlight using informative outliers to regularize models \cite{chen2021atom,ming2022poem,zhu2023DivOE}. They often employ OE and Energy-based learning objectives and fail to utilize correlations between ID and OOD data in regularization. In addition, \cite{jingkang2021iccvwild,katz2022wilds} explore a more complex setting where auxiliary datasets could contain ID noise. Strong ID attributes exist here. Our method also shows good robustness in this case.

\noindent\textbf{Multi-view Data Model (MVDM).}
\cite{allen-zhu2023towards} presents the MVDM to explore ensemble learning. 
It assumes a natural data point has multiple views that can be used for classification. A single neural network trained with gradient descent might classify a data point based on only one of those views. Ensembling multiple models or distilling ensemble models into one model could uncover all features. Previous research in OOD detection with auxiliary outliers usually characterizes ID and OOD data in latent space using a simplified Gaussian-like model \cite{chen2021atom}. 
This model makes it hard to understand the correlations between ID and OOD data concretely. 
Instead, we extend MVDM to provide a systematic understanding of their correlations.

\section{Preliminaries on Multi-view Data Model}
We first give the formal in-distribution definition by restating the original Multi-view Data Model \cite{allen-zhu2023towards} in OOD detection context. 
Moreover, since MVDM contains two typical settings, \textit{i.e.}, single model and ensemble distillation, we will review them and briefly summarize how models perform feature learning in these two settings.

\subsection{In-distribution Definition and Notations} 
We consider a $k$-class classification setting in OOD detection. Any data input is represented by $P$ patches $X = (x_1,x_2,\dots, x_P) \in \mathbb {R}^{d \times P}$, where each patch has dimension $d$. For ID data, labels belong to $[k]$. It is assumed that there are multiple associated features for each label $ j \in [k] $, say two features for simplicity, represented by two \textit{orthogonal unit feature vectors} $ v_{j,1}, v_{j,2} \in \mathbb{R}^d$.

Our in-distribution definition is adapted from the original MVDM for OOD detection.
$D^{in}$ is defined via the multi-view distribution $D_m^{in}$ and single-view distribution $D_s^{in}$. $D_m^{in}$ represents images that we can observe all main ID features and use any of these features to classify them. $D_s^{in}$ represents images taken from a particular angle, where one or more main ID features can be missing, \textit{i.e.},
\begin{definition}[data distributions \( D_m^{in} \) and \( D_s^{in} \)]
Given \( D \in \{D_m^{in}, D_s^{in}\} \), we define \( (X^{in}, y) \sim D\) as follows. First, choose the label \( y \in [k] \) uniformly at random. Then, \( X^{in} \) is generated as follows:\\
    \textbf{1.} Denote \(\mathcal{V}(X^{in}) = \{v_{y,1}, v_{y,2}\} \cup \mathcal{V}'\) as the set of feature vectors used in this data vector \(X\), where \(\{v_{y,1},v_{y,2}\}\) are \textbf{main ID features} and \(\mathcal{V}'\) is a set of \textbf{minor ID features} uniformly sampled from \(\{v_{j,1}, v_{j,2}\}_{j \in [k] \setminus \{y\}}\).\\
    \textbf{2.} For each \( v \in \mathcal{V}(X) \), pick 
    many disjoint patches in \( [P] \) and denote them as \( \mathcal{P}_v(X^{in}) \subset [P] \). 
    We denote \( \mathcal{P}(X^{in}) = \bigcup_{v \in \mathcal{V}(X^{in})} \mathcal{P}_v(X^{in}) \).\\
    \textbf{3.} If \( D = D_s^{in} \) is the single-view distribution, pick a value \( \hat{\ell} = \ell(X^{in}) \in [2] \) uniformly at random.\\
    \textbf{4.} For each \( v \in \mathcal{V}(X^{in}) \) and \( p \in \mathcal{P}_v(X^{in}) \), \( x_p = z_p v + \) ``noise'' \( \in \mathbb{R}^d \). These random coefficients \( z_p \geq 0 \) satisfy:\\
             \textbf{$\cdot$} In the case of multi-view distribution \(D = D_m^{in}\):\\
             1) \( \sum_{p \in \mathcal{P}_v(X^{in})} z_p \in [1, O(1)] \) when \( v \in \{v_{y,1}, v_{y,2}\}\); \\
             2) \( \sum_{p \in \mathcal{P}_v(X^{in})} z_p \in [\Omega(1), 0.4] \) when \( v \in V(X^{in}) \setminus \{v_{y,1}, v_{y,2}\}\).\\
             \textbf{$\cdot$} In the case of single-view distribution \( D = D_s^{in} \):\\
             1) \( \sum_{p \in \mathcal{P}_v(X^{in})} z_p \in [1, O(1)] \) when picked \( v = v_{y, \hat{\ell}} \);\\
             2) \( \sum_{p \in \mathcal{P}_v(X^{in})} z_p\) is much smaller than that of \(v_{y, \hat{\ell}}\) and can be ignored when \( v \in V(X^{in}) \setminus \{v_{y, \hat{\ell}}\} \).\\
    \textbf{5.} For each \( p \in P\setminus P(X^{in}) \), \( x_p \) consists only of ``noise''.
\end{definition}
\begin{definition}[$D^{in}$ and $Z^{in}$]
The final in-distribution $D^{in}$ consists of data from $D_m^{in}$ w.p. $\mu$ and $D_s^{in}$ w.p. $1-\mu$. We are given $N$ training samples from $D_{in}$, and denote the training data set as $Z^{in} = Z_m^{in} \cup Z_s^{in}$ where $Z_m^{in}$ and $Z_s^{in}$ respectively represent multi-view and single-view training data. We write $(X^{in}, y) \sim Z^{in}$ as $(X^{in}, y)$ sampled randomly from $Z^{in}$.
\end{definition}
\label{learningnetwork}
Furthermore, a simplified neural network is provided to conduct analyses on MVDM.
Concretely, this neural network is represented by a tow-layer convolutional network \( F(X) = (F_1(X), \ldots, F_k(X)) \in \mathbb{R}^k \) parameterized by \( w_{i,r} \in \mathbb{R}^d \). For \( i \in [k], r \in [m] \), $w_{i,r}$ satisfies:
\[
\forall i \in [k]: F_i(X) = \sum_{r \in [m]} \sum_{p \in [P]} \widetilde{\text{ReLU}}(\langle w_{i,r}, x_p \rangle),
\]
where $\widetilde{\text{ReLU}}$ is a smoothed activation function. 
More details are given in Appendix A. 

\subsection{
Feature-Learning based on MVDM
}
To briefly introduce the learning mechanism under single model and ensemble distillation model settings, a thought experiment is utilized in \cite{allen-zhu2023towards}. 
Here, we present its basic settings and main conclusions to facilitate the understanding of our work.

\noindent\textbf{(1) Basic Settings:} Given $k=2$ and four ``features" \(v_1, v_2, v_3, v_4\). \( v \in \{v_1, v_2\} \) is associated with label 1, and \( v{'} \in \{v_3, v_4\} \) is associated with 2. Main and minor ID features weigh 1 and 0.1, respectively. For each category, there are $80\%$ training data from $D_m^{in}$. The remaining $20\%$ are from $D_s^{in}$, \textit{i.e.}, one half has one main ID feature, and the other half has the second.
Data points with labels 1 and 2 could have minor ID features \( v{'} \) and $v$, respectively. 

\noindent\textbf{(2) Main Conclusions:}
\textbf{(i) Single model} can learn only one feature of each category.
While training a single neural network with random initialization using cross-entropy loss,
the network will pick up one feature for each label and correctly classify $90\%$ training examples, \textit{i.e.}, $80\%$ multi-view data and $10\%$ single-view data with the picked feature. The remaining $10\%$ examples will be memorized using noise in the data. Thus, this single model learns to classify each category using only one feature.
\textbf{(ii) Ensemble model} can learn both features of each category.
Depending on the random initialization, each network will pick up \( v_1 \) or \( v_2 \), each with a probability of $50\%$. Therefore, as long as we ensemble many independently trained models, with a high probability their ensemble will pick up all features \( v_1, v_2, v_3, v_4 \).
\textbf{(iii) Ensemble distillation model} can learn both features of each category.
Ensemble distillation \cite{hinton2015distilling} trains an individual model to match the ensemble models' outputs. The key idea is that the model can learn all features via the soft labels produced by ensemble models.

\section{Method}
In this section, we propose a structured multi-view-based out-of-distribution learning framework (\textbf{MVOL}) to tackle OOD detection via mining ID attributes in outliers. 
\subsection{Extended Multi-view Data Model}
We extend MVDM to study the correlations between ID and OOD data formally. OOD data is assumed to mainly consist of minor ID features and noise, \textit{i.e.},
\begin{definition}[Out-of-distribution $D^{out}$]We define $X^{out} \sim D^{out}$ as follows. $X^{out}$ is generated by:\\
    \textbf{1}. Denote $\mathcal{V}(X^{out})$ as the set of \textbf{minor ID feature} vectors used in this data vector $X^{out}$, which are uniformly sampled from $\{v_{j , 1}, v_{j, 2}\}_{{j}\in{[k]}}$.\\
    \textbf{2}. For each $v \in \mathcal{V}(X^{out})$, pick many disjoint patches in $[P]$ and denote it as $P_v(X^{out}) \subset [P]$. We denote $P(X^{out}) = \bigcup_{v\in \mathcal{V}(X^{out})} \mathcal{P}_v(X)$.\\
    \textbf{3}. For each $v \in \mathcal{V}(X^{out})$ and $p \in \mathcal{P}_v(X^{out})$, we set \( x_p = z_p v\ + \) ``noise'' \( \in \mathbb{R}^d \).
    These random coefficients $z_p \geq 0$ satisfy that:
        $\sum_{p \in \mathcal{P}_v(X^{out})} z_p \in [\Omega(1), 0.4]$. \\
    \textbf{4}. For each $p \in [P] \setminus \mathcal{P}(X^{out})$, $x_p$ consists only of ``noise”.
\end{definition}
\begin{definition}[$Z^{out}$]
We are given M auxiliary OOD training samples from $D^{out}$, and denote the training data set as $Z^{out}$. We write $X^{out} \sim Z^{out}$ as $X^{out}$ sampled randomly from $Z^{out}$. We are given samples $M \geq N$ to represent a large auxiliary OOD dataset.
\end{definition}
\subsection{New Insight on MaxLogit as an OOD Score}
\label{maxlogitasagoodoodscore}
MVOL employs MaxLogit as the OOD scoring function. 
MaxLogit measures ID attributes contained in test inputs. Outliers with minor ID features tend to have a low score. We justify MaxLogit with our extended MVDM.

\noindent\textbf{MaxLogit based OOD detector.} For a sample $(X^{in}, y) \sim D^{in}$ or $ X^{out} \sim D^{out}$ and a neural network $F$. Feeding $X \in \{X^{in}, X^{out}\}$ into $F$, we get logit outputs \( F(X) = (F_1(X), \ldots, F_k(X)) \in \mathbb{R}^k \). Then, the MaxLogit scoring function is given as follows.
\[\text{MaxLogit}(X; F) = \text{max} (F_1(X), \ldots, F_k(X)).\]
Then MaxLogit can be used in the following OOD detector:
\begin{equation}
\label{OODdetector}
G(X;\tau, F) = 
\begin{cases} 0 & \text{if } \text{MaxLogit}(X; F) \leq \tau, \\
1 & \text{if } \text{MaxLogit}(X; F) > \tau,
\end{cases}
\end{equation}
where 0 and 1 represent OOD and ID by convention, respectively. The threshold $\tau$ is chosen so that a high fraction of ID data (\textit{e.g.}, 95\%) is with the MaxLogit above $\tau$. 

\noindent\textbf{Theoretical analysis with our extended MVDM.}

\noindent\textbf{(1) Interpretability of MaxLogit:} We demonstrate that MaxLogit can be an interpretable ID/OOD indicator, supported by our extended MVDM stated above. We define:
\begin{align}
&\ I(X) = \text{argmax}_{j \in [k]} \sum_{p \in P_{v_{j, 1}}(X) \bigcup P_{v_{j, 2}}(X)} z_p;\\ 
&\ 
z(X) = \text{max}_{j \in [k]} \sum_{p \in P_{v_{j, 1}}(X) \bigcup P_{v_{j, 2}}(X)} z_p 
\end{align}
$I(X)$ is the category with the largest sum of coefficients on associated features. $z(X)$ is this sum value.
\begin{proposition}
\label{proposition1}
For every $X^{out} \sim D^{out}$, every $(X^{in}_s, y_s) \sim D_{s}^{in}$, and every $(X^{in}_m, y_m) \sim D_{m}^{in}$, we have:
\[z(X^{out}) < z(X^{in}_s) \quad\text{and}\quad z(X^{out}) < z(X^{in}_m)\]
\end{proposition}
\begin{assumption}
\label{assumption1}
For every $i,j \in [k]$ and every $l,l^{'} \in [2]$, consider an ideal classifier $F$ satisfying
\[\sum_{r \in [m]} [\langle w_{i,r}, v_{i,l} \rangle ]^+ = \sum_{r \in [m]} [\langle w_{j,r}, v_{j,l^{'}} \rangle]^+ \pm o(1)\]
\end{assumption}
\noindent Note that throughout this paper $[a]^+ = \max\{0, a\}$.

\textbf{Intuition:} When Assumption \ref{assumption1} is valid, we can deduce that $F_{I(X^{out})}({X^{out}}) < F_{I(X^{in}_s)}(X^{in}_s) $ and $F_{I(X^{out})}({X^{out}}) < F_{I(X^{in}_m)}(X^{in}_m)$ with Proposition \ref{proposition1}, which corresponds to relation among MaxLogit scores. In other words, the MaxLogit of ID data can be higher than that of OOD data. Therefore, MaxLogit serves as an interpretable OOD scoring function
(more details in Appendix B). 

\noindent\textbf{(2) Efficacy of MaxLogit:} 
We provide error-bound analysis in the single model setting, where models learn only part of features for classification, and in the ensemble distillation model setting, where models uncover all features via distilling knowledge from ensemble models. These analyses reveal the efficacy of MaxLogit.
\begin{assumption}
\label{assumption2}
For every $i,j \in [k]$, every $l,l^{'} \in [2]$ and $v_{i,l}, v_{j,l'}$ are learned by single model, consider an ideal classifier $F$ satisfing: 
\[\sum_{r \in [m]} [\langle w_{i,r}, v_{i,l} \rangle ]^+ = \sum_{r \in [m]} [\langle w_{j,r}, v_{j,l^{'}} \rangle]^+ \pm o(1)\]
\end{assumption}

\begin{theorem}[Calibrated Single Model]
\label{theorem1}
Suppose we train a single model 
from random initialization. 
When Assumption \ref{assumption2} is valid, we have:
\[
    FNR(F) \leq \frac{1}{2}(1 - \mu +o(1))
\]
\end{theorem}

\begin{theorem}[Calibrated Ensemble Distillation Model]
\label{theorem2}
Suppose we train a model using ensemble distillation. 
When Assumption \ref{assumption1} is valid, we have:
\[
    FNR(F) \leq o(1)
\]
\end{theorem}
Theorem \ref{theorem1} tells us \textbf{i)} the error bound in the single model setting can be established with the multi-view data proportion $\mu$ in ID datasets and \textbf{ii)} this bound decreases as $\mu$ increases. Theorem \ref{theorem2} shows that the error bound can reach near zero in the ensemble distillation model setting, where models learn all features. It highlights the superiority of MaxLogit as an OOD score.
Refer to Experiments and Appendix B for empirical and proof details, respectively.
\begin{figure*}[htbp]
\centering
\includegraphics[scale=0.88]{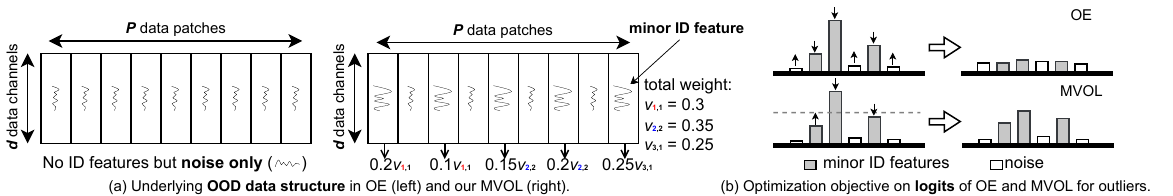}
\caption{(a) Illustration of OOD data structure assumed in OE and our MVOL. (b) From our new assumption on out-of-distribution,
OE blindly aligns the logits of all categories to the same level. Instead, MVOL adaptively aligns the logits of categories with associated minor ID features on outliers. 
The key insight is that a well-calibrated model should not display significantly distinct activation on these categories since minor ID features usually have small coefficients.
}

\label{assuption behind UniP and FGP}
\end{figure*}
\subsection{Multi-view-based Learning Objective}
Models trained solely on in-distribution data could be uncalibrated and produce unexpected high MaxLogit for out-of-distribution inputs \cite{wei2022logitNorm}. We calibrate models' logits for OOD inputs with auxiliary outliers.

In this section, we first use confidence loss in OE as an example to analyze the common problem in current outlier exposure-based methods, \textit{i.e.}, overlooking intrinsic ID attributes in outliers. Then a multi-view-based learning objective is proposed to utilize these attributes explicitly. 


\noindent\textbf{Revisiting OE via our extended MVDM}\\ 
\label{Revisiting UniP via Our Extended MVDM}
\noindent\textbf{(1) Confidence loss.} 
Previous studies typically interpret softmax-normalized logits as probability vectors summing to one. Based on this assumption, confidence loss in OE constrains the predictive distribution of auxiliary outliers, forcing it to approximate a uniform distribution, \textit{i.e.},
\begin{equation}
\begin{aligned}
\mathcal{L}_{\text{OE}}=& \frac{1}{N}\sum_{j=1}^N -\log P_{\theta}(\hat{y}=y |X_j^{in}) + \\
&\frac{\beta}{M}\sum_{j=1}^M \sum_{i=1}^k -\frac{1}{k}\log P_{\theta}(\hat{y} = i |X_j^{out})
\end{aligned}
\label{UniP(1}
\end{equation}
\noindent where $y$ is the one-hot label of ID sample $X^{in}_j$, $\beta$ is a loss weight for confidence-loss term and $\frac{1}{k}$ is the label on each category for OOD sample $X^{out}_j$. 

\noindent\textbf{(2) Pitfalls of OE on logit calibration.} 
Our extended MVDM offers a logit calibration lens to revisit OE's underlying assumption.
As shown in Figure \ref{assuption behind UniP and FGP} (a) left, OE assumes outliers only contain random noise. Consequently, OE with confidence loss penalizes models to output uniform softmax-normalized logits, and produce approximate logits for all categories, as shown in Figure \ref{assuption behind UniP and FGP} (b) above.  
However, outliers might contain minor ID features that should not be ignored, as shown in Figure \ref{assuption behind UniP and FGP} (a) right that depicts a more realistic OOD data structure. 
Although unexpected high logits on outliers could be suppressed, confidence loss treats all categories equally, ignoring the inherent discrepancy between the categories only with noise (type I) and those with minor ID features (type II) in outliers.
As a result, OE would overemphasize noise and underemphasize minor ID features when aligning the logits of all categories to the same level.

\noindent\textbf{Multi-view based learning objective.}
\label{ProposedFGP}
Given the typical problem of OE, we introduce a new learning objective. In Figure \ref{assuption behind UniP and FGP} (b) below, MVOL adaptively aligns the logits of categories with associated minor ID features on outliers, \textit{i.e.},
\begin{align}
\label{fgpLoss}
\mathcal{L}^{(t)}_{\text{MVOL}} &=  \frac{1}{N}\sum_{j=1}^N -\log P_{\theta}(\hat{y}=y |X_j^{in}) \\
&+ \frac{\beta}{M}\sum_{j=1}^M \sum_{i=1}^k - p_{j,i}^{(t)}\log P_{\theta}(\hat{y} = i |X_j^{out}),\nonumber \\
& \text{where} \quad p_{j,i}^{(t)} = \mathbf{min}(\mathbf{logit}_i(F^{(t)},X_j^{out}), \epsilon),
\label{softlabelofFGP}
\end{align}
\(\mathbf{logit}_i(F, X) \overset{\text{def}}{=} \frac{e^{F_i(X)}}{\sum_{n\in[k]}e^{F_n(X)}}\). We take \(S_{j, I}^{(t)} = \{i\in [k]| \mathbf{logit}_i(F^{(t)}, X_j^{out}) \leq \epsilon \}\), \(S_{j, II}^{(t)} \overset{\text{def}}{=} \{i\in [k] | \mathbf{logit}_i(F^{(t)}, X_j^{out}) > \epsilon \}\) to represent those categories with smaller softmax-normalized values than threshold $\epsilon$ or not on outlier $X_j$.  
In Equation (\ref{softlabelofFGP}), if $\mathbf{logit}_i(F^{(t)}, X_j^{out})$ exceeds $\epsilon$, $p_{j, i}$ is set to $\epsilon$. Conversely, if $\mathbf{logit}_i(F^{(t)}, X^{out}_j)$ is less than $\epsilon$, $p_{j,i}$ adopts $\mathbf{logit}_i(F^{(t)}, X_j^{out})$ directly. 
$p_{j,i}$ can be cast as the soft label adopted in Equation (\ref{fgpLoss}). The rationale behind our learning objective will be further explained through the parameter update mechanism, \textit{i.e.},
\begin{equation}
\label{updateofFGP}
\begin{aligned} 
&w_{i,r}^{(t+1)}= w_{i,r}^{(t)} - \eta \nabla_{w_{i,r}} L(F^{(t)}) - \frac{\eta^{'}}{M}\sum_{j=1}^M \Big[\\
\Big(&\mathbf{logit}_i(F^{(t)},X^{out}_j) \sum^{k}_{n=1}p_{j,n}^{(t)}  -p_{j, i}^{(t)}\Big) \nabla_{w_{i,r}} F^{(t)}_i(X^{out}_j)\Big].
\end{aligned}
\end{equation}
where $L(F^{(t)}$ means a cross-entropy loss for ID data.

\begin{table*}[ht]
  \caption{Main results on out-of-distribution detection. The best result is in bold and the second is underlined. Values are averaged over five runs. Post Hoc-based methods use the same five pre-trained models but apply different scoring functions.}
  \centering
  {\small
    \tabcolsep=0.25cm
    \begin{tabular}{p{0.4cm}<{\centering}l|l|lll|lll}
    \toprule
     & \multirow{2}[0]{*}{\textbf{Category}}& \multirow{2}[0]{*}{\textbf{Method}} & \multicolumn{3}{c|}{\underline{\ \ \ \textbf{CIFAR-10}\ \ \ }} & \multicolumn{3}{c}{\underline{\ \ \ \textbf{CIFAR-100}\ \ \ }} \\
     & & & FPR95 $\downarrow$ & AUROC $\uparrow$ & ID-Acc $\uparrow$& FPR95 $\downarrow$& AUROC $\uparrow$& ID-Acc $\uparrow$\\
    \midrule
    \parbox[c]{0mm}{\multirow{13}{*}{\rotatebox[origin=c]{90}{{\textbf{Single Model Setting}}}}} & \multirow{7}[0]{*}{Post Hoc} & MSP & 56.29{\tiny \ $\pm$ 1.62} & 89.59{\tiny \ $\pm$ 0.63} & 94.27{\tiny \ $\pm$ 0.14} & 80.75{\tiny \ $\pm$ 0.81} & 74.70{\tiny \ $\pm$ 1.01} & \underline{74.69}{\tiny \ $\pm$ 0.21} \\
    & & Energy & 41.20{\tiny \ $\pm$ 5.28} & 89.70{\tiny \ $\pm$ 1.93} & 94.27{\tiny \ $\pm$ 0.14} & 72.58{\tiny \ $\pm$ 1.78} & 79.01{\tiny \ $\pm$ 1.12} & 74.69{\tiny \ $\pm$ 0.21} \\
    & & MaxLogit & 41.68{\tiny \ $\pm$ 4.99} & 89.69{\tiny \ $\pm$ 1.88} & 94.27{\tiny \ $\pm$ 0.14} & 73.21{\tiny \ $\pm$ 1.69} & 78.88{\tiny \ $\pm$ 1.11} & 74.69{\tiny \ $\pm$ 0.21} \\
    & & ODIN & 41.75{\tiny \ $\pm$ 3.86} & 87.38{\tiny \ $\pm$ 2.41} & 94.27{\tiny \ $\pm$ 0.14} & 68.13{\tiny \ $\pm$ 1.83} & 79.36{\tiny \ $\pm$ 0.91} & 74.69{\tiny \ $\pm$ 0.21} \\
    & & Mahalanobis & 23.96{\tiny \ $\pm$ 1.26} & 92.81{\tiny \ $\pm$ 0.32} & 94.27{\tiny \ $\pm$ 0.14} & 46.40{\tiny \ $\pm$ 3.73} & 87.44{\tiny \ $\pm$ 1.12} & 74.69{\tiny \ $\pm$ 0.21} \\
    & & KNN & 30.89{\tiny \ $\pm$ 2.76} &   94.53{\tiny \ $\pm$ 0.44} & 94.27{\tiny \ $\pm$ 0.14}& 82.02{\tiny \ $\pm$ 2.58} & 75.84{\tiny \ $\pm$ 1.35} & 74.69{\tiny \ $\pm$ 0.21}\\
    & & ASH & 40.03{\tiny \ $\pm$ 5.18} & 90.01{\tiny \ $\pm$ 1.70} & 94.26{\tiny \ $\pm$ 0.11} & 63.31{\tiny \ $\pm$ 1.91} & 79.35{\tiny \ $\pm$ 1.09} & 74.23{\tiny \ $\pm$ 0.31}\\
    \cmidrule(ll){2-9}
    & \multirow{2}[0]{*}{Outlier Synthesis} & VOS & 34.67{\tiny \ $\pm$ 5.01} & 91.54{\tiny \ $\pm$ 1.92} & \textbf{94.75}{\tiny \ $\pm$ 0.17} & 70.17{\tiny \ $\pm$ 2.52} & 81.73{\tiny \ $\pm$ 1.78} & \textbf{75.94}{\tiny \ $\pm$ 0.20}\\
    & & ATOL& 12.86{\tiny \ $\pm$ 0.59} & 97.34{\tiny \ $\pm$ 0.07} & 93.89{\tiny \ $\pm$ 0.17} & 64.67{\tiny \ $\pm$ 1.73} & 80.17{\tiny \ $\pm$ 1.34} & 72.70{\tiny \ $\pm$ 0.17}\\
    \cmidrule(ll){2-9}
    & \multirow{4}[0]{*}{Outlier Exposure} &
     Energy w/Aux & 4.70{\tiny \ $\pm$ 0.50} & 97.77{\tiny \ $\pm$ 0.06} & 90.74{\tiny \ $\pm$ 0.24} & 52.43{\tiny \ $\pm$ 3.51} & 88.40{\tiny \ $\pm$ 1.16} & 62.13{\tiny \ $\pm$ 0.27} \\
    & & OE & 4.25{\tiny \ $\pm$ 0.15} & 98.56{\tiny \ $\pm$ 0.07} & 94.47{\tiny \ $\pm$ 0.13} & 46.51{\tiny \ $\pm$ 3.65} & 89.78{\tiny \ $\pm$ 0.98} & 74.02{\tiny \ $\pm$ 0.04} \\
    & & OE + MaxLogit & \underline{4.12}{\tiny \ $\pm$ 0.20} & \underline{98.58}{\tiny \ $\pm$ 0.07} & 94.47{\tiny \ $\pm$ 0.13} & \underline{46.20}{\tiny \ $\pm$ 3.53} & \underline{90.59}{\tiny \ $\pm$ 0.87} & 74.02{\tiny \ $\pm$ 0.04} \\
    \rowcolor{\rowbg} \cellcolor{white} & \cellcolor{white} &\textbf{MVOL} {\scriptsize(ours)}  & \textbf{3.30}{\tiny \ $\pm$ 0.19} & \textbf{98.70}{\tiny \ $\pm$ 0.05} & \underline{94.68}{\tiny \ $\pm$ 0.09} & \textbf{42.96}{\tiny \ $\pm$ 0.86} & \textbf{90.69}{\tiny \ $\pm$ 0.26} & 74.29{\tiny \ $\pm$ 0.33} \\
    \midrule
    \midrule
    \parbox[c]{0mm}{\multirow{13}{*}{\rotatebox[origin=c]{90}{{\textbf{Ensemble Distillation Model Setting}}}}} & \multirow{7}[0]{*}{Post Hoc} & MSP & 55.82{\tiny \ $\pm$ 2.46} & 89.52{\tiny \ $\pm$ 0.53} & 94.36{\tiny \ $\pm$ 0.11} & 80.36{\tiny \ $\pm$ 0.72} & 74.72{\tiny \ $\pm$ 0.79} & \underline{76.99}{\tiny \ $\pm$ 0.15} \\
    & & Energy & 38.23{\tiny \ $\pm$ 1.72} & 89.97{\tiny \ $\pm$ 0.57} & 94.36{\tiny \ $\pm$ 0.11} & 71.97{\tiny \ $\pm$ 1.02} & 79.66{\tiny \ $\pm$ 0.57} & 76.99{\tiny \ $\pm$ 0.15} \\
    & & MaxLogit & 38.64{\tiny \ $\pm$ 2.00} & 89.94{\tiny \ $\pm$ 0.57} & 94.36{\tiny \ $\pm$ 0.11} & 72.71{\tiny \ $\pm$ 1.08} & 79.46{\tiny \ $\pm$ 0.60} & 76.99{\tiny \ $\pm$ 0.15} \\
    & & ODIN & 40.46{\tiny \ $\pm$ 2.20} & 86.94{\tiny \ $\pm$ 1.07} & 94.36{\tiny \ $\pm$ 0.11} & 67.71{\tiny \ $\pm$ 0.61} & 78.71{\tiny \ $\pm$ 0.78} & 76.99{\tiny \ $\pm$ 0.15} \\
    & & Mahalanobis & 23.06{\tiny \ $\pm$ 0.64} & 92.94{\tiny \ $\pm$ 0.19} & 94.36{\tiny \ $\pm$ 0.11} & \underline{42.36}{\tiny \ $\pm$ 1.93} & 88.39{\tiny \ $\pm$ 0.49} & 76.99{\tiny \ $\pm$ 0.15} \\
    & & KNN & 41.98{\tiny \ $\pm$ 2.47} & 92.70{\tiny \ $\pm$ 0.55} & 94.36{\tiny \ $\pm$ 0.11} & 84.67{\tiny \ $\pm$ 2.72} & 73.67{\tiny \ $\pm$ 1.78} & 76.99{\tiny \ $\pm$ 0.15} \\
    & & ASH  & 36.50{\tiny \ $\pm$ 1.43} & 91.55{\tiny \ $\pm$ 0.55} & 94.23{\tiny \ $\pm$ 0.09} & 59.19{\tiny \ $\pm$ 1.75} & 80.20{\tiny \ $\pm$ 0.49} & 76.41{\tiny \ $\pm$ 0.26} \\
    \cmidrule(ll){2-9}
    & \multirow{2}[0]{*}{Outlier Synthesis} & VOS & 30.58{\tiny \ $\pm$ 4.34} & 92.23{\tiny \ $\pm$ 1.07} & \textbf{95.02}{\tiny \ $\pm$ 0.11} & 72.01{\tiny \ $\pm$ 1.81} & 79.86{\tiny \ $\pm$ 1.90} & \textbf{77.18}{\tiny \ $\pm$ 0.28}\\
    & & ATOL & 28.14{\tiny \ $\pm$ 1.79} & 93.60{\tiny \ $\pm$ 0.43} & 94.19{\tiny \ $\pm$ 0.07} & 74.07{\tiny \ $\pm$ 0.98} & 77.82{\tiny \ $\pm$ 0.66} & 74.79{\tiny \ $\pm$ 0.09}\\
    \cmidrule(ll){2-9}
    & \multirow{4}[0]{*}{Outlier Exposure}&
     Energy w/Aux & 4.10{\tiny \ $\pm$ 0.24} & 98.07{\tiny \ $\pm$ 0.04} & 91.48{\tiny \ $\pm$ 0.19} & 52.81{\tiny \ $\pm$ 3.52} & 89.14{\tiny \ $\pm$ 0.81} & 68.27{\tiny \ $\pm$ 0.48} \\
    & & OE & 3.95{\tiny \ $\pm$ 0.23} & 98.56{\tiny \ $\pm$ 0.07} & 94.67{\tiny \ $\pm$ 0.23} & 47.04{\tiny \ $\pm$ 0.73} & 89.35{\tiny \ $\pm$ 0.21} & 75.01{\tiny \ $\pm$ 0.13} \\
    & & OE + MaxLogit & \underline{3.61}{\tiny \ $\pm$ 0.24} & \textbf{98.62}{\tiny \ $\pm$ 0.06} & 94.67{\tiny \ $\pm$ 0.23} & 46.92{\tiny \ $\pm$ 0.75} & \textbf{90.79}{\tiny \ $\pm$ 0.26} & 75.01{\tiny \ $\pm$ 0.13} \\
    \rowcolor{\rowbg} \cellcolor{white} & \cellcolor{white} & \textbf{MVOL} {\scriptsize(ours)} & \textbf{3.34}{\tiny \ $\pm$ 0.20} & \underline{98.61}{\tiny \ $\pm$ 0.06} & \underline{94.68}{\tiny \ $\pm$ 0.20} & \textbf{36.62}{\tiny \ $\pm$ 1.36} & \underline{90.37}{\tiny \ $\pm$ 0.43} & 76.27{\tiny \ $\pm$ 0.33} \\
    \bottomrule
    \end{tabular}%
    }
  \label{mainresults}%
\end{table*}%

\noindent\textbf{(1) Detailed analysis.} Based on our data model, type II categories usually exhibit much larger logits than type I categories. \textbf{(i)} $S^{(t)}_{j, II}$ is generally associated with type II categories. By uniformly setting their labels to the threshold $\epsilon$, our learning objective can moderate the logits of these categories to ensure a smoother logit response, based on the insight that a well-calibrated model should not display significantly distinct activation values since minor ID features often have smaller coefficients. 
This allows our learning objective to adaptively adjust the logit values for type II categories rather than blindly suppressing them via aligning them to logits of type I categories. \textbf{(ii)} $S^{(t)}_{j, I}$ is often associated with type I categories. By setting $p_{j, i}$ $i \in S_{j, I}^{(t)}$ as the softmax-normalized value itself, the gradient weight in Equation (5) can be written as a non-positive value \((\sum_{n=1}^{k} p_{j,n}^{(t)} - 1)p_{j, i}\), where the magnitude is proportional to $p_{j, i}$. Here, a lower $p_{j, i}$ leads to a smaller weight, thus contributing less to parameter updates. Thus, our $\mathcal{L}_\text{MVOL}$ effectively reduces the excessive focus on noise. 
More analysis is in Appendix E.

\noindent\textbf{(2) Adaptability to ID noise in wild datasets.}
When auxiliay datasets contain ID data as noise, i.e., wild datasets \cite{jingkang2021iccvwild,katz2022wilds}, eliminating all ID noise is labor-intensive.
Benefiting from explicitly utilizing ID attributes in outliers, our learning objective demonstrates strong adaptability to ID data in wild datasets. 
In specific, when encountering ID data with ground truth label $y$ in wild datasets, models tend to produce a significantly larger logit for the category $y$ than the others. This reaction causes the coefficient \(\sum_{n=1}^{k} p_{j,n}^{(t)}\) in Equation (\ref{updateofFGP}) to be small, reducing its impact on the logit of $y$, rather than blindly suppressing it. 
Thus, our learning objective can generalize to ID data in wild datasets. Relevant experiments are in Table \ref{noise oe}.

\begin{table*}[ht]
  \centering
  \caption{Main results on OOD detection with wild datasets (a larger $\alpha$ means more ID noise). WOODS is for reference.}
    {\small
    \tabcolsep=0.25cm
    \begin{tabular}{l|l|lll|lll}
    \toprule
    \multirow{2}[0]{*}{\textbf{Noise}} &\multirow{2}[0]{*}{\textbf{Method}} & \multicolumn{3}{c|}{\underline{\ \ \ \textbf{Single Model Setting}}\ \ \ } & \multicolumn{3}{c}{\underline{\ \ \ \textbf{Ensemble Distillation Model Setting}\ \ \ }} \\
           & & FPR95 $\downarrow$ & AUROC $\uparrow$& ID-Acc $\uparrow$   & FPR95 $\downarrow$ & AUROC $\uparrow$& ID-Acc $\uparrow$ \\
    \midrule
    -& MaxLogit  & 47.39{\tiny \ $\pm$ 2.93} & 89.81{\tiny \ $\pm$ 0.55} & 91.38{\tiny \ $\pm$ 0.24} & 40.59{\tiny \ $\pm$ 2.50} & 90.21{\tiny \ $\pm$ 0.71} & 91.94{\tiny \ $\pm$ 0.09} \\ 
    \cmidrule{1-8}
    
    \multirow{3}[0]{*}{$\alpha = 0$}& WOODS & \textcolor[RGB]{128,128,128}{21.97{\tiny \ $\pm$ 1.83}} & \textcolor[RGB]{128,128,128}{96.02{\tiny \ $\pm$ 0.32}} & \textcolor[RGB]{128,128,128}{91.20{\tiny \ $\pm$ 0.22}} & \textcolor[RGB]{128,128,128}{51.50{\tiny \ $\pm$ 3.99}} & \textcolor[RGB]{128,128,128}{84.85{\tiny \ $\pm$ 0.99}} & \textcolor[RGB]{128,128,128}{89.79{\tiny \ $\pm$ 0.19}}\\
    & OE + MaxLogit & 18.04{\tiny \ $\pm$ 1.34} & \textbf{96.57}{\tiny \ $\pm$ 0.15} & \textbf{92.24}{\tiny \ $\pm$ 0.08} & 18.86{\tiny \ $\pm$ 2.10} & 96.12{\tiny \ $\pm$ 0.31} & \textbf{92.32}{\tiny \ $\pm$ 0.15} \\
    \rowcolor{\rowbg} \cellcolor{white} &\textbf{MVOL} {\scriptsize(ours)} & \textbf{17.34}{\tiny \ $\pm$ 2.86} & 96.21{\tiny \ $\pm$ 0.30} & 91.71{\tiny \ $\pm$ 0.08} & \textbf{12.96}{\tiny \ $\pm$ 0.95} & \textbf{96.45}{\tiny \ $\pm$ 0.14} & 91.96{\tiny \ $\pm$ 0.13} \\
    \cmidrule{1-8}
    
    \multirow{3}[0]{*}{$\alpha = 0.05$} 
    &WOODS & \textcolor[RGB]{128,128,128}{22.04{\tiny \ $\pm$ 2.36}} & \textcolor[RGB]{128,128,128}{96.02{\tiny \ $\pm$ 0.34}} & \textcolor[RGB]{128,128,128}{91.27{\tiny \ $\pm$ 0.16}} & \textcolor[RGB]{128,128,128}{51.49{\tiny \ $\pm$ 3.97}} & \textcolor[RGB]{128,128,128}{84.86{\tiny \ $\pm$ 0.99}} & \textcolor[RGB]{128,128,128}{89.79{\tiny \ $\pm$ 0.19}}\\
    & OE + MaxLogit & 22.11{\tiny \ $\pm$ 1.26} & 95.64{\tiny \ $\pm$ 0.26} & 91.29{\tiny \ $\pm$ 0.20} & 23.11{\tiny \ $\pm$ 3.90} & 95.67{\tiny \ $\pm$ 0.48} & 91.51{\tiny \ $\pm$ 0.18} \\
    \rowcolor{\rowbg} \cellcolor{white} &\textbf{MVOL} {\scriptsize(ours)} & \textbf{19.55}{\tiny \ $\pm$ 1.04} & \textbf{96.22}{\tiny \ $\pm$ 0.16} & \textbf{91.75}{\tiny \ $\pm$ 0.23} & \textbf{12.87}{\tiny \ $\pm$ 1.11} & \textbf{96.49}{\tiny \ $\pm$ 0.11} & \textbf{91.74}{\tiny \ $\pm$ 0.30} \\
    \cmidrule{1-8}
    
    \multirow{3}[0]{*}{$\alpha = 0.1$} & WOODS & \textcolor[RGB]{128,128,128}{22.38{\tiny \ $\pm$ 2.30}} & \textcolor[RGB]{128,128,128}{95.98{\tiny \ $\pm$ 0.35}} & \textcolor[RGB]{128,128,128}{91.27{\tiny \ $\pm$ 0.21}} & \textcolor[RGB]{128,128,128}{51.59{\tiny \ $\pm$ 4.03}} & \textcolor[RGB]{128,128,128}{84.85{\tiny \ $\pm$ 0.98}} & \textcolor[RGB]{128,128,128}{89.81{\tiny \ $\pm$ 0.19}}\\
    & OE + MaxLogit & 25.49{\tiny \ $\pm$ 1.60} & 94.97{\tiny \ $\pm$ 0.30} & 90.92{\tiny \ $\pm$ 0.31} & 26.90{\tiny \ $\pm$ 3.04} & 95.24{\tiny \ $\pm$ 0.41} & 91.01{\tiny \ $\pm$ 0.22} \\
    \rowcolor{\rowbg} \cellcolor{white} & \textbf{MVOL} {\scriptsize(ours)} & \textbf{18.05}{\tiny \ $\pm$ 1.58} & \textbf{96.16}{\tiny \ $\pm$ 0.20} & \textbf{91.55}{\tiny \ $\pm$ 0.31} & \textbf{13.96}{\tiny \ $\pm$ 0.80} & \textbf{96.07}{\tiny \ $\pm$ 0.15} & \textbf{91.64}{\tiny \ $\pm$ 0.28} \\
    \cmidrule{1-8}

    \multirow{3}[0]{*}{$\alpha = 0.3$}& WOODS & \textcolor[RGB]{128,128,128}{22.03{\tiny \ $\pm$ 2.45}} & \textcolor[RGB]{128,128,128}{95.99{\tiny \ $\pm$ 0.40}} & \textcolor[RGB]{128,128,128}{91.24{\tiny \ $\pm$ 0.19}} & \textcolor[RGB]{128,128,128}{51.58{\tiny \ $\pm$ 4.03}} & \textcolor[RGB]{128,128,128}{84.86{\tiny \ $\pm$ 0.99}} & \textcolor[RGB]{128,128,128}{89.80{\tiny \ $\pm$ 0.21}}\\
    & OE + MaxLogit & 37.67{\tiny \ $\pm$ 3.85} & 92.64{\tiny \ $\pm$ 0.44} & 89.04{\tiny \ $\pm$ 0.30} & 51.20{\tiny \ $\pm$ 6.17} & 91.83{\tiny \ $\pm$ 0.57} & 88.60{\tiny \ $\pm$ 0.13} \\
    \rowcolor{\rowbg} \cellcolor{white} & \textbf{MVOL} {\scriptsize(ours)} & \textbf{20.71}{\tiny \ $\pm$ 1.86} & \textbf{95.94}{\tiny \ $\pm$ 0.32} & \textbf{91.24}{\tiny \ $\pm$ 0.16} & \textbf{13.79}{\tiny \ $\pm$ 0.93} & \textbf{96.43}{\tiny \ $\pm$ 0.16} & \textbf{91.76}{\tiny \ $\pm$ 0.07} \\
    \cmidrule{1-8}
    
    \multirow{3}[0]{*}{$\alpha = 0.5$} & WOODS & \textcolor[RGB]{128,128,128}{\textbf{22.31}{\tiny \ $\pm$ 3.05}} & \textcolor[RGB]{128,128,128}{\textbf{95.93}{\tiny \ $\pm$ 0.50}} & \textcolor[RGB]{128,128,128}{\textbf{91.28}{\tiny \ $\pm$ 0.12}} & \textcolor[RGB]{128,128,128}{51.49{\tiny \ $\pm$ 3.96}} & \textcolor[RGB]{128,128,128}{84.86{\tiny \ $\pm$ 0.99}} & \textcolor[RGB]{128,128,128}{89.80{\tiny \ $\pm$ 0.20}}\\
    & OE + MaxLogit & 45.14{\tiny \ $\pm$ 2.94} & 90.22{\tiny \ $\pm$ 0.63} & 88.04{\tiny \ $\pm$ 0.18} & 54.23{\tiny \ $\pm$ 3.77} & 89.93{\tiny \ $\pm$ 0.21} & 86.80{\tiny \ $\pm$ 0.21} \\
    \rowcolor{\rowbg} \cellcolor{white} & \textbf{MVOL} {\scriptsize(ours)} & 25.53{\tiny \ $\pm$ 1.23} & 95.23{\tiny \ $\pm$ 0.18} & 90.79{\tiny \ $\pm$ 0.21} & \textbf{14.85}{\tiny \ $\pm$ 1.86} & \textbf{96.44}{\tiny \ $\pm$ 0.32} & \textbf{91.81}{\tiny \ $\pm$ 0.19} \\
    \bottomrule
    \end{tabular}%
    }
  \label{noise oe}%
\end{table*}

\section{Experiments}

\subsection{Experimental Setup}
\label{Experimental Setup}
\textbf{Datasets.}~\textbf{(1)}~
Following the common benchmarks in literature, we use CIFAR-10 and CIFAR-100 as ID datasets, and six diverse OOD test sets
(details in Appendix E). 
The Tiny Images dataset used in \cite{hendrycks2018oe,liu2020energy} for auxiliary outliers has been withdrawn due to its ethical wrong.
Instead, we use 300K RandomImages, a cleaned subset of the Tiny Images dataset by \cite{hendrycks2018oe}, as done in \cite{katz2022wilds}.  
\textbf{(2)}~For experiments with wild datasets, we partition CIFAR-10 into 2 halves: one provides ID training data, and the other provides auxiliary noisy ID data. CIFAR-100 provides auxiliary outliers. The auxiliary ID and OOD data are mixed with varied noise levels $\alpha$, \textit{i.e.}, 0, 0.05, 0.1, 0.3, 0.5, and the total images in auxiliary datasets are maintained at 50000. This setup allows simulating the scenarios, where auxiliary OOD datasets contain considerable noise yet most are still outliers. 

\noindent\textbf{Evaluation Metrics.}
We evaluate methods using common metrics, averaged on six OOD test sets, i.e.,
the false positive rate of declaring OOD examples as ID when 95\% of ID data points are declared as ID (FPR95) and the area under the receiver operating characteristic curve (AUROC).

\noindent\textbf{OOD Detection Baselines.}
We compare MVOL with comprehensive baselines. \textbf{(1)} Post-Hoc based: MSP \cite{hendrycks2017MSP}, MaxLogit\cite{hendrycks2022maxlogit}, ODIN \cite{shiyu2018odin}, Mahalanobis \cite{lee2018MAHA}, and Energy \cite{liu2020energy}; \textbf{(2)} outlier synthesis-based: VOS \cite{VOS} and ATOL \cite{ATOL}, \textbf{(3)} outlier exposure-based: OE \cite{hendrycks2018oe} and Energy \cite{liu2020energy}. Moreover, OE with MaxLogit as the OOD score, dubbed OE + MaxLogit, is our most relevant baseline. In wild-dataset settings, the SOTA WOODS is \cite{katz2022wilds} for reference. 

\noindent\textbf{Training Details.} \textbf{(1)} We use the Wide ResNet with 40 layers and a widen factor of 2 for all experiments. \textbf{(2)} Involved two experimental settings: (i) \textit{Single model}: single neural network is trained with one-hot labels and cross entropy loss. (ii) \textit{Ensemble distillation model}: a single network is trained with the standard ensemble distillation algorithm in \cite{hinton2015distilling}, where soft labels are generated by ensemble models. 
More details are in Appendix.

\subsection{Results}
\label{section_mainresults}
\textbf{Main results on OOD detection.}
In this experiment, Outlier Exposure-based methods utilize auxiliary OOD datasets.
Through comparisons in Table \ref{mainresults}, we have four main findings. \textbf{(1)} OE + MaxLogit outperforms OE (MSP as OOD scoring function), usually serving as the best baseline. This supports our analysis with our extended MVDM, \textit{i.e.}, OE performs logit calibration on outliers. \textbf{(2)} Our MVOL obtains an overall stronger performance than other methods. In the ensemble distillation model setting, MVOL can reduce FPR95 by $10.3\%$ on CIFAR-100, with a slight decrease of 0.04\% in AUROC. \textbf{(3)} With CIFAR-100 as an ID dataset, compared to MVOL in the single model setting, MVOL in the ensemble distillation model setting can reduce FPR95\% by 6.34\%. This result supports the comparison of Theorem 1 and 2. It also indicates the superiority of MaxLogit when models learns all features.
There is a small improvement on CIFAR-10. The reason could be that a single model learns more features on CIFAR-10 than CIFAR-100, as CIFAR-10 has more samples per class \cite{allen-zhu2023towards}.
\textbf{(4)} Across both settings on two ID datasets, MVOL often degrades less ID accuracy than other OE-based methods.

Furthermore, we analyze that MVOL can utilize ID attributes in outliers as follows.
\textbf{(1)} In Figure \ref{vis} (a), compared to OE + MaxLogit, MVOL not only preserves ID features on OOD data but also mitigates overemphasizing noise. It verified our key insight shown in Figure \ref{assuption behind UniP and FGP} (b). 
\textbf{(2)} 
In Table \ref{mainresults}, we observe that MVOL generally shows greater superiority in FPR95 than AUROC. 
Based on this, we further analyze MVOL focus on ID attributes in outliers. 
In Figure \ref{vis} (b), compared to OE + MaxLogit, MVOL achieves a lower FPR at a high TPR, e.g., 0.95. 
This suggests that MVOL has better discriminability of ID against OOD, as a high TPR means a high recall of ID samples. 
In a unified perspective on ID and OOD, logit responses to OOD data would affect ID data.
This discriminability could due to MVOL not blindly suppressing logit responses to ID attributes in outliers. 
Instead, MVOL adaptively calibrates these logits. As a result, it reserves logit responses to ID data and improves discriminability with MaxLogit. 
In addition, the smaller FPR of OE + MaxLogit, when the TPR is low, could be due to overfitting a subset of ID data against OOD, increasing AUROC.

\begin{figure}[htbp]
\centering
\includegraphics[scale=0.35]{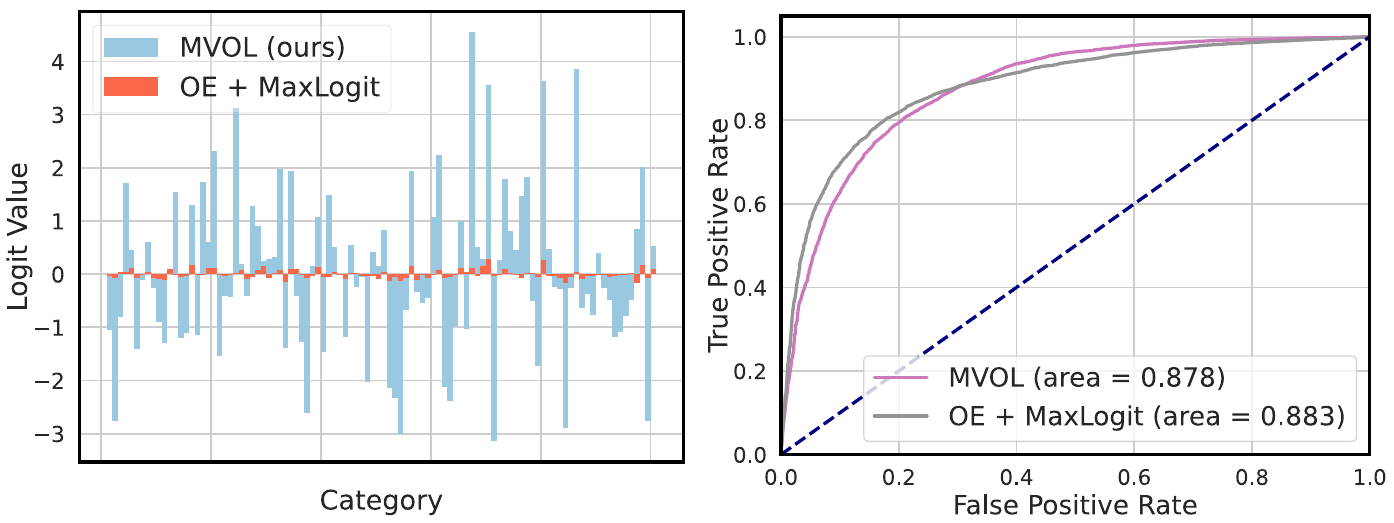}
\caption{
MVOL and OE + MaxLogit with CIFAR-100 as ID dataset. (a) Logits of a training OOD sample, (b) ROC curve with Textures as the test OOD dataset.
}
\label{vis}
\end{figure}

\noindent\textbf{Main results on OOD detection with wild datasets.}
This experiment simulates a real scenario where collected auxiliary OOD datasets contain ID data as noise. We use OE + MaxLogit as the baseline for comparison. Because compared to it head-to-head, our MVOL has a new multi-view-based learning objective ($\mathcal{L}_\text{MVOL}$) specially designed to tackle intrinsic ID attributes in outliers. We use wild datasets to verify whether $\mathcal{L}_\text{MVOL}$ could generalize well to noisy ID data in them. Furthermore, we use WOODS as the SOTA baseline in wild-dataset settings only for reference, as it is specifically designed to handle noisy ID data in wild dataset.

We have two main findings in Table \ref{noise oe}.
\textbf{(1)} MVOL shows advantages using wild datasets in both settings. For example, in single model setting, $\mathcal{L}_\text{MVOL}$ reduces FPR95 by 0.7\%, 2.56\%, 7.44\%, 16.96\%, and 19.61\% as alpha increases from 0 to 0.5, showing increasingly clear advantages over OE + MaxLogit, which confirms its generalizability.
\textbf{(2)} MVOL is comparable with the customized WOODS when alpha=0.5 in single model setting. WOODS degrades in ensemble distillation model setting, even worse than MaxLogit (without auxiliary data). The reason could be its inefficacy in benefiting from distilled knowledge of ensemble models.

\section{Conclusion}
This paper explores the intrinsic ID attributes in outliers. We propose a structured multi-view-based out-of-distribution learning framework to handle these attributes rationally. We provide theoretical analysis on it. Extensive experiments on auxiliary OOD datasets and even wild datasets show its efficacy. Through our unified perspective on ID and OOD, 
we probably calibrate logits even without auxiliary outliers, reducing computation costs. We leave it as future works.

\section{Acknowledgments}
This work is supported by the National Natural Science Foundation of China (NSFC) under Grant No. 62476049.

\bibliography{aaai}

\newpage
\appendix
\onecolumn

\section{A. Full Version of Data Distribution and Notations}
\label{Full Version of Data Distribution and Notations}

Here, we provide the full version of in-distribution and out-of-distribution definitions, refered to \cite{allen-zhu2023towards}. 

We consider OOD detection in k-class classification scenarios. Each input is represented by $P$ patches $X = (x_1,x_2,\dots, x_P) \in \mathbb {R}^{d \times P}$, where each patch has dimension $d$. For ID data, labels belong to $[k]$. We use \( \widetilde{O}, \widetilde{\Omega}, \widetilde{\Theta} \) notations to hide polylogarithmic factors in \( k \). It is assumed that there are multiple associated features for each label class \( j \in [k] \), say two features for simplicity, represented by unit feature vectors \( v_{j,1}, v_{j,2} \in \mathbb{R}^d \). We denote by:
\[\mathcal{V} \stackrel{\text{def}}{=} \{v_{j,1}, v_{j,2} \mid j \in [k]\} \text{ the set of all features.}\]
Let $C_p$ be a global constant, $s \in [1, k^{0.2}]$ be a global parameter to control feature sparsity, $\sigma_p = \frac{1}{\sqrt{d}\operatorname{polylog}(k)}$ be a parameter to control magnitude of the \textit{random noise}, $\gamma = \frac{1}{k^{1.5}}$ controls the \textit{feature noise}. In this detailed definition, we allow "noise" to be any feature noise plus Gaussian noise.
\begin{Def}[data distributions \(D_m^{in}\) and \(D_s^{in}\)] Given \(D \in \{D_m^{in}, D_s^{in}\}\), we define \((X^{in}, y) \sim D\) as follows. First choose the label \(y \in [k]\) uniformly at random. Then, the data vector \(X^{in}\) is generated as follows.\\
\textbf{1.} Denote \(\mathcal{V}(X^{in}) = \{v_{y,1}, v_{y,2}\} \cup \mathcal{V}'\) as the set of feature vectors used in this data vector \(X^{in}\), where \(\{v_{y,1},v_{y,2}\}\) are \textbf{main ID features} and \(\mathcal{V}'\) is a set of \textbf{minor ID features} uniformly sampled from \(\{v_{j,1}, v_{j,2}\}_{j \in [k] \setminus \{y\}}\), each with probability \(\frac{s}{k}\).\\
{\footnotesize\(\diamond\) comment: \((X^{in}, y)\) shall be primarily supported on two main features \(v_{y,1}, v_{y,2}\) and \(\sim O(s)\) minor features}\\
\textbf{2.} For each \(v \in \mathcal{V}(X^{in})\), pick \(C_p\) many disjoint patches in \([P]\) and denote it as \(P_v(X^{in}) \subset [P]\) (the distribution of these patches can be arbitrary). We denote \(\mathcal{P}(X^{in}) = \bigcup_{v \in \mathcal{V}(X^{in})} P_v(X^{in})\).\\
{\footnotesize\(\diamond\) comment: the weights of \(X^{in}\) on each feature \(v\) shall be written on patches in \(P_v(X^{in})\)}\\
\textbf{3.} If \(D = D_s^{in}\) is the single-view distribution, pick a value \(\hat{\ell} = \hat{\ell}(X^{in}) \in [2]\) uniformly at random.\\
\textbf{4.} For each \(v \in \mathcal{V}(X^{in})\) and \(p \in P_v(X^{in})\), we set
\[
x_p = z_p v + \sum_{v' \in \mathcal{V}} \alpha_{p,v'} v' + \xi_p \in \mathbb{R}^d
\]
Above, each \(\alpha_{p,v'} \in [0, \gamma]\) is the feature noise, and \(\xi_p \sim \mathcal{N}(0, \sigma_p^2 I)\) is an (independent) random Gaussian noise. The coefficients \(z_p \geq 0\) satisfy that:\\
In the case of multi-view distribution \(D = D_m^{in}\),\\
\(\sum_{p \in P_v(X)} z_p \in [1, O(1)]\) when \(v \in \{v_{y,1}, v_{y,2}\}\),\\
\(\sum_{p \in P_v(X)} z_p \in [\Omega(1), 0.4]\) when \(v \in \mathcal{V}(X) \setminus \{v_{y,1}, v_{y,2}\}\),\\
\(\diamond\) comment: total weights on features \(v_{y,1}, v_{y,2}\) are larger than those on minor features \(\mathcal{V}(X) \setminus \{v_{y,1}, v_{y,2}\}\).\\
In the case of single-view distribution \(D = D_s^{in}\),\\
\(\sum_{p \in P_v(X)} z_p \in [1, O(1)]\) when \(v = v_{y, \hat{\ell}}\),\\
\(\sum_{p \in P_v(X)} z_p \in [\rho, O(\rho)]\) when \(v = v_{y, 3 - \hat{\ell}}\) \(\diamond\) comment: we consider \(\rho = k^{-0.01}\) for simplicity,\\
\(\sum_{p \in P_v(X^{in})} z_p \in [\Omega(\Gamma), \Gamma]\) when \(v \in \mathcal{V}(X^{in}) \setminus \{v_{y,1}, v_{y,2}\}\). {\footnotesize we consider \(\Gamma = \frac{1}{\text{polylog}(k)}\) for simplicity.}\\
{\footnotesize\(\diamond\) comment: total weight on feature \(v_{y, \hat{\ell}}\) is much larger than those on \(v_{y, 3 - \hat{\ell}}\) or minor features.}\\
\textbf{5.} For each \(p \in [P] \setminus P(X^{in})\), we set:
\[
x_p = \sum_{v' \in \mathcal{V}} \alpha_{p,v'} v' + \xi_p
\]
where \(\alpha_{p,v'} \in [0, \gamma]\) is the feature noise and \(\xi_p \sim \mathcal{N}(0, \frac{\gamma^2 k^2}{d} \textbf{I})\) is (independent) Gaussian noise.
\end{Def}

\begin{Def}[$D^{in}$ and $Z^{in}$]
The final In-distribution $D^{in}$ consists of data from $D_m^{in}$ w.p. $\mu$ and $D_s^{in}$ w.p. $1-\mu$. We are given $N$ training samples from $D^{in}$, and denote the training data set as $Z^{in} = Z_m^{in} \cup Z_s^{in}$ where $Z_m^{in}$ and $Z_s^{in}$ respectively represent multi-view and single-view training data. We write $(X^{in}, y) \sim Z^{in}$ as $(X^{in}, y)$ sampled uniformly at random from the empirical data set, and denote $N_s = |Z_s^{in}|$. We again for simplicity focus on the setting when \(\mu = \frac{1}{\text{poly}(k)}\) and we are given samples \(N = k^{1.2}/\mu\) so each label \(i\) appears at least \(\Omega(1)\) in \(Z_s^{in}\).
\end{Def}

\begin{Def}[Out-of-distribution $D^{out}$]We define $X^{out} \sim D^{out}$ as follows. The data vector $X^{out}$ is generated as follows:\\
    \textbf{1}. Denote $\mathcal{V}(X^{out})$ as the set of \textbf{minor ID feature} vectors used in this data vector $X^{out}$, which are uniformly sampled from $\{v_{j , 1}, v_{j, 2}\}_{{j}\in{[k]}}$, each with probability $\frac{s}{k}$. \\
    \textbf{2)}. For each $v \in \mathcal{V}(X^{out})$, pick $C_p$ many disjoint patches in $[P]$ and denote it as $P_v(X^{out}) \subset [P]$ (the
    distribution of these patches can be arbitrary). We denote $P(X^{out}) = \bigcup_{v\in \mathcal{V}(X^{out})} \mathcal{P}_v(X)$.\\
    $\cdot$  {\footnotesize comment: the weights of $X^{out}$ on each feature $v$ shall be written on patches in $\mathcal{P}_v(X^{out})$ }\\
    \textbf{3)}. For each $v \in \mathcal{V}(X^{out})$ and $p \in \mathcal{P}_v(X^{out})$, we set 
    \[ x_p = z_p v + \sum_{v^{'} \in \mathcal{V}} a_{p,v^{'}} v' + \xi_p \in \mathbb{R}^d \]
    Above, each $a_{p,v^{'}} \in [0,\gamma]$ is the feature noise, and $\xi_p \sim \mathcal{N}(0, \sigma_p^2\mathbf{I})$ is an (independent) random Gaussian noise. The coefficients $z_p \geq 0$ satisfy that:
    \[\sum_{p \in \mathcal{P}_v(X^{out})} z_p \in [\Omega(1), 0.4]\]
    \textbf{4)}. For each $p \in[\mathcal{P}] \backslash \mathcal{P}(X)$, we set:
    \[
    x_p = \sum_{v' \in \mathcal{V}} a_{p,v'} v' + \xi_p
    \]
\end{Def}

\begin{Def}[$Z^{out}$]
We are given M auxiliary OOD training samples from $D^{out}$, and denote the training data set as $Z^{out}$. We write $X^{out} \sim Z^{out}$ as $X^{out}$ sampled uniformly at random from the empirical data set. We are given samples $M \geq N$ to represent a large auxiliary OOD dataset.
\end{Def}

\begin{Def}[$\widetilde{\text{ReLU}}$]
For integer \( q \geq 2 \) and threshold \( \lambda = \frac{1}{\text{polylog}(k)} \), the smoothed activation function is defined as:
\end{Def}
\[
    \widetilde{\text{ReLU}}(z) = 
    \begin{cases} 
    0 & \text{if } z \leq 0, \\
    \frac{z^q}{q \lambda^{q-1}} & \text{if } z \in [0, \lambda], \\
    z - \left(1 - \frac{1}{q}\right) \lambda & \text{if } z \geq \lambda.
    \end{cases}
\]

\section{B. Theoretical Proof}
\label{Theoretical Proof}
\subsection{Preliminary}
We first list some useful conclusions in \citep{allen-zhu2023towards} to complete our proof.

\subsubsection{Single Model Setting.}
The proof relies on an induction hypothesis for iteration $t = 0, 1, 2, \dots, T$. Before stating it, we introduce several notations. Let us denote
\begin{align*}
\Lambda_i^{(t)} \overset{\text{def}}{=}& \max_{r \in [m], \ell \in [2]} [ w_{i,r}^{(t)}, v_{i,\ell} \rangle ]^+ \quad \\ \quad \Lambda_{i,\ell}^{(t)} \overset{\text{def}}{=}& \max_{r \in [m]} [ \langle w_{i,r}^{(t)}, v_{i,\ell} \rangle ]^+
\end{align*}

Suppose $m \leq \text{poly}(k)$. For every $i \in [k]$, let us denote
\begin{align*}
    \mathcal{M}_i^{(0)} \overset{\text{def}}{=} \Big\{ \in [m] &\mid \exists \ell \in [2] : \\
    &\langle w_{i,r}^{(0)}, v_{i,\ell}\rangle  \geq \lambda_{i,\ell}^{(0)} \left(1 - O(\frac{1}{\log k})\right) \Big\}
\end{align*}

\textbf{Justification.} 
In the single model setting, if a neuron $r \in [m]$ is not in $\mathcal{M}_i^{(0)}$, those neurons $r$ will not learn anything useful for the output label $ i \in [k]$.
\begin{Fact}
With probability at least $1 - e^{-\Omega(\log^5 k)}$, we have $|M_i^{(0)}| \leq m_0 \overset{\text{def}}{=} O(\log^5 k)$.
\end{Fact}

Suppose we denote by

\begin{align*}
S_{i,\ell} \overset{\text{def}}{=} \mathbb{E}_{(X,y)\sim Z_m} \Big[\mathbf{1}_{y=i}\sum_{p \in P_{v_{i,l}}(X)} z_p^q\Big] 
\end{align*} 

Then, define  
\begin{align*}
\mathcal{M}  \overset{\text{def}}{=}& \Big\{ (i, \ell^\ast) \in [k] \times [2] \mid \Lambda_{i, \ell^\ast}^{(0)} \geq\\
&\Lambda_{i, 3 - \ell^\ast}^{(0)}\left( \frac{S_{i,3-\ell^\ast}}{S_{i,\ell^\ast}} \right)^{\frac{1}{q-2}} \Big( 1 + \frac{1}{\log^2(m)} \Big) \Big\}
\end{align*}
\textbf{Justification.} If $(i,\ell) \in M$, the feature $v_{i,\ell}$ has a higher chance than $v_{i,3-\ell}$ to be learned by the model.
\begin{Prop}Suppose $m \leq \text{poly}(k)$. We have the following properties about $\mathcal{M}$:\\
    \textbf{$\cdot$} For every $i \in [k]$, at most one of $(i, 1)$ or $(i, 2)$ is in $\mathcal{M}$.\\
    \textbf{$\cdot$} For every $i \in [k]$, suppose $S_{i,e} \geq S_{i,3-\ell}$, then $\Pr\left((i,3-\ell) \in M\right) \geq m^{-O(1)}$.\\
    \textbf{$\cdot$} For every $i \in [k]$, $\Pr\left((i,1) \in M \text{ or } (i,3 - \ell) \in \mathcal{M}\right) \geq 1 - o(1)$.
\end{Prop}
\textbf{Justification.} With decent probability at least one of $(i, 1)$ or $(i, 2)$ shall be in $\mathcal{M}$. More importantly, the later Induction Hypothesis \ref{induction hypothesisB.1} ensures that, during the entire training process, if $(i, 3 - \ell) \in \mathcal{M}$, then $v_{i,\ell}$ must be missing from the learner network.

\begin{Hypo}
\label{induction hypothesisB.1}
For every \(\ell \in [2]\), for every \(r \in [m]\), for every \((X, y) \in Z_m^{in}\), and for every \(i \in [k]\), or for every \((X, y) \in Z_s^{in}\) and \(i \in [k] \backslash\ \{y\}\):\\
    (a) For every \(p \in P_{v_i,\ell} (X, y)\), we have:
    \(
    \langle w^{(t)}_{i, r}, x_p\rangle = \langle w^{(t)}_{i, r}, v_{i,\ell}\rangle z_p \pm \widetilde{o}(\sigma_0).
    \)\\
    (b) For every \(p \in P(X) \setminus (P_{v_{i,1}} (X, y) \cup P_{v_{i,2}} (X, y))\), we have:
    \(
    |\langle w^{(t)}_{i, r}, x_p \rangle| \leq \widetilde{O}(\sigma_0).
    \)\\
    (c) For every \(p \in [P] \setminus P(X)\), we have:
    \(
    |\langle w^{(t)}_{i, r}, x_p \rangle| \leq \widetilde{O}(\sigma_0 \gamma k).
    \)\\
    In addition, for every $(X,y) \in Z_s^{in}$, every $i \in [k]$, every $r \in [m]$, every $\ell \in [2]$\\
    (d) For every \(p \in P_{v_{i, \ell}}(X, y)\), we have:
    \(
    \langle w^{(t)}_{i, \ell}, x_p \rangle = \langle w^{(t)}_{i, \ell}, x_p\rangle + \langle w^{(t)}_{i, \ell}, x_p\rangle \pm \widetilde{O}(\sigma_0 \gamma k).
    \)\\
    (e) For every \(p \in P_{v_{i, \ell}}(X, y)\), if \((i, 3 - \ell) \in \mathcal{M}\) we have:
    \(
    |\langle w^{(t)}_{i, r}, x_p \rangle| \leq \widetilde{O}(\sigma_0).
    \)\\
    (f) For every \(p \in P_{v_{i, \ell}}(X, y)\), if \(r \in [m] \backslash \mathcal{M}^{(0)}_i\) we have:
    \(
    |\langle w^{(t)}_{i, r}, x_p \rangle| \leq \widetilde{O}(\sigma_0).
    \)\\
    Moreover, we have for every \(i \in [k]\),\\
    (g) \(
    \Lambda^{(t)}_i \geq \Omega(\sigma_0) \quad \text{and} \quad \Lambda^{(t)}_i \leq \widetilde{O}(1).
    \)\\
    (h) For every \(\ell \in [m]\), every \(\ell \in [2]\), it holds that:
    \(
    \langle w^{(t)}_{i, \ell} x_p\rangle \geq -\widetilde{O}(\sigma_0).
    \)\\
    (i) For every \(\ell \in [m] \backslash \mathcal{M}_i^{(0)}\), every \(\ell \in [2]\), it holds that:
    \(
    \langle w^{(t)}_{i, \ell}, v_{i,l}\rangle \leq \tilde{O}(\sigma_0).
    \)
\end{Hypo}

\textbf{Justification.} 
Items (a)-(c) say that when studying the correlation between \(w_{i,r}\) with a multi-view data, or between \(w_{i,r}\) with a single-view data (but \(y \neq i\)), the correlation is about \(\langle w_{i,r}, v_{i,1}\rangle\) and \(\langle w_{i,r}, v_{i,2} \rangle\) and the remaining terms are sufficiently small. Items (d)-(f) say that when studying the correlation between \(w_{i,r}\) with a single-view data \((X, y)\) with \(y = i\), then the correlation also has a significant noise term \(\langle w^{(t)}_{i,r}, \xi_p \rangle\). Items (g)-(i) are some regularization statements.

We denote
\[
\Phi^{(t)}_{i,\ell} \overset{\text{def}}{=} \sum_{r \in [m]} \left[(w^{(t)}_{i,r}, v_{i,\ell})^{+}\right] \quad \text{and} \quad \Phi^{(t)}_i \overset{\text{def}}{=} \sum_{\ell \in [2]} \Phi^{(t)}_{i,\ell} 
\] 

\subsubsection{Ensemble Distillation Model Setting.}
Similarly, in this setting, useful conclusions for our proof are as follows,
\begin{Hypo}
For every \(\ell \in [2]\), for every \(r \in [m]\), for every \((X, Y) \in Z\) and \(i \in [k]\),
(a) For every \(p \in P_{v_{i,\ell}}(X)\), we have:
    \(
    \langle w^{(t)}_{i, r}, x_p \rangle = \langle w^{(t)}_{i, r}, v_{i,\ell}\rangle z_p \pm \widetilde{o}(\sigma_0).
    \)\\
(b) For every \(p \in P(X) \setminus (P_{v_{i,1}}(X) \cup P_{v_{i,2}}(X))\), we have:
    \(
    |\langle w^{(t)}_{i, r},x_p\rangle| \leq \widetilde{O}(\sigma_0).
    \)\\
(c) For every \(p \in [P] \setminus P(X)\), we have:
    \(
    |\langle w^{(t)}_{i, r}, x_p\rangle| \leq \widetilde{O}(\sigma_0 \gamma k).
    \)\\
Moreover, we have for every \(i \in [k]\), every \(\ell \in [2]\),\\
(g) 
    \(
    \Phi^{(t)}_{i,\ell} \geq \Omega(\sigma_0) \quad \text{and} \quad \Phi^{(t)}_{i,\ell} \leq \widetilde{O}(1).
    \)\\
(h) For every \(r \in [m]\), it holds that:
    \(
    \langle w^{(t)}_{i, r}, v_{i,\ell} \rangle \geq -\widetilde{O}(\sigma_0).
    \)
\end{Hypo}
\subsection{Main Analysis}
Based on the above preliminary, with our extended multi-view data model for OOD detection, we can get the following Induction Hypothesis for both single model and ensemble distillation model settings.
\begin{Hypo}
\label{Induction Hypothesis B.2}
For every $\ell \in [2]$, for every $r \in [m]$, for every $X^{out} \in D^{out}$ and $i\in[k]$,\\
(a) For every \(p \in P_{v_{i,\ell}}(X)\), we have:
    \(
    \langle w^{(t)}_{i, r}, x_p \rangle = \langle w^{(t)}_{i, r}, v_{i,\ell}\rangle z_p \pm \widetilde{o}(\sigma_0).
    \)\\
(b) For every \(p \in [P] \setminus P(X)\), we have:
    \(
    |\langle w^{(t)}_{i, r}, x_p\rangle| \leq \widetilde{O}(\sigma_0 \gamma k).
    \)\\
\end{Hypo}
We consider there are two stages in OOD detection with Outlier Exposure. (1) Feature learning. In this stage, models in single model and ensemble distillation model settings can learn features in their own way. (2) Logit calibration. With auxiliary OOD datsets, models can be calibrated in logit level.
\begin{Assum}
\label{assum1}
For every $i,j \in [k]$ and every $l,l^{'} \in [2]$, consider an ideal classifier $F$ satisfying
\[\sum_{r \in [m]} [\langle w_{i,r}, v_{i,l} \rangle ]^+ = \sum_{r \in [m]} [\langle w_{j,r}, v_{j,l^{'}} \rangle]^+ \pm o(1)\]
\end{Assum}

\begin{Assum}
\label{Assum2}
For every $i,j \in [k]$, every $l,l^{'} \in [2]$ and $v_{i,l}, v_{j,l'}$ are learned by single model, consider an ideal classifier $F$ satisfying: 
\[\sum_{r \in [m]} [\langle w_{i,r}, v_{i,l} \rangle ]^+ = \sum_{r \in [m]} [\langle w_{j,r}, v_{j,l^{'}} \rangle]^+ \pm o(1)\]
\end{Assum}

\begin{Prop}
\label{prop1}
For every $X^{out} \sim D^{out}$, every $(X^{in}_s, y_s) \sim D_{s}$, and every $(X^{in}_m, y_m) \sim D_{m}$, we have:
\[z(X^{out}) < z(X^{in}_s) \quad\text{and}\quad z(X^{out}) < z(X^{in}_m)\]
\end{Prop}
Suppose Induction Hypothesis \ref{Induction Hypothesis B.2} holds, based on the definition of $D^{out}$, $D^{in}_s$ and $D^{in}_m$, we know that \(\Omega(1)\leq z(X^{out}) \leq 0.8\), \(1 \leq z(X^{in}_s) \leq O(1)\) and \(2 \leq z(X^{in}_m) \leq O(1)\). Therefore, we can get \(z(X^{out}) < z(X^{in}_s)\) and \(z(X^{out}) < z(X^{in}_m)\). When Assumption \ref{assum1} is satisfied, 
\begin{align*}
F_{I(X^{out})}(X^{out}) &= \sum_{p \in P_{v_{I(X^{out}), 1}}(X^{out}) \bigcup P_{v_{I(X^{out}), 2}}(X^{out})}\\
& z_p\langle w^{(t)}_{I(X^{out}), r}, v_{j,\ell}\rangle \pm \widetilde{o}(\sigma_0).
\end{align*}
We can get similar results on $F_{I(X^{in}_s)}(X^{in}_s)$ and $F_{I(X^{in}_m)}(X^{in}_m)$. Combined with Proposition \ref{prop1}, we can get that 
\begin{align*}
F_{I(X^{out})}(X^{out}) &< F_{I(X^{in}_s)}(X^{in}_s), \\ F_{I(X^{out})}(X^{out}) 
&< F_{I(X^{in}_m)}(X^{in}_m)
\end{align*}
which supports the ID/OOD discriminability of MaxLogit.

Suppose Induction Hypothesis \ref{Induction Hypothesis B.2} holds, we can prove Theorem 1 and 2 as follows.
\begin{Theom}[Calibrated Single Model]
Suppose we train a single model using the gradient descent update starting from the random initialization. When Assumption \ref{Assum2} is valid, we have:
\[
    FNR(F) \leq \frac{1}{2}(1 - \mu +o(1))
\]
\end{Theom}

\begin{proof}
For \(i, j \in [k]\), \(\ell^{'}_i, \ell^{'}_j \in \mathcal{M}\), according to Assumption \ref{Assum2}, we have
\[|\Phi_{i, \ell^{'}_i} - \Phi_{j, \ell^{'}_j}| = \widetilde{o}(\sigma_0) \]
we denote,
\[ \Phi_{i,\ell_i^{'}} = \Phi_{j, \ell_j^{'}} = C \pm \widetilde{o}(\sigma_0).\]

For \((X^{in}_s, y_s) \in D^{in}_s\) with half probability \(\ell(X^{in}_s) \notin \mathcal{M}\), we have  \(F_y(X) \leq O(\rho) + \frac{1}{\text{ploylog}(k)}\).

For \((X^{in}_s, y_s) \in D^{in}_s\) with half probability \(\ell(X^{in}_s) \in \mathcal{M}\),  we have  \(F_y(X) \geq C \pm \widetilde{O}(\sigma_0)\).

For \((X^{in}_m, y_m) \in D^{in}_m\), we have \(F_y(X) \geq C \pm \widetilde{O}(\sigma_0)\).

For \(X^{out} \in D^{out}\), we have \(F_{I(X^{out})}(X^{out}) \leq 0.4C \pm \widetilde{O}(\sigma_0)\).

As In-distribution $D_{in}$ consists of data from $D_m^{in}$ w.p. $\mu$ and $D_s^{in}$ w.p. $1 - \mu$., we can immediately get the \(FNR(F) \leq \frac{1}{2}(1-\mu + o(1))\).
\end{proof}
Theorem 1 reveals a calibrated single model can discriminate all multi-view ID data and half of single-view ID data from OOD. 

\begin{Theom}[Calibrated Ensemble Distillation Model]
Suppose we train a model using ensemble distillation. When Assumption \ref{assum1} is valid, we have:
\[
    FNR(F) \leq o(1)
\]
\end{Theom}
\begin{proof}
For \(i, j \in [k]\), models can learn both features of categories $i,j$ in the ensemble distillation model setting. According to Assumption \ref{assum1}, for every $\ell_i \in [2]$ and every $\ell_j \in [2]$, we have
\[|\Phi_{i, \ell_i} - \Phi_{j, \ell_j}| = \widetilde{o}(\sigma_0) \] 

we denote

\[\quad \Phi_{i,\ell_i} = \Phi_{j, \ell_j} = C^{'} \pm \widetilde{o}(\sigma_0).\]

For \((X^{in}_s, y_s) \in D^{in}_s\), we have  
\(F_y(X^{in}_s) \geq C^{'} \pm \widetilde{O}(\sigma_0)\).

For \((X^{in}_m, y_m) \in D^{in}_m\), we have \(F_y(X^{in}_m) \geq 2C' \pm \widetilde{O}(\sigma_0)\).

For \(X^{out} \in D^{out}\), we have \(F_{I(X^{out})}(X^{out}) \leq 0.4C' \pm \widetilde{O}(\sigma_0)\).
Therefore, we can easily get \(F_{I(X^{out})}(X^{out}) < F_y(X^{in}_s)\) and \(F_{I(X^{out})}(X^{out}) < F_y(X^{in}_m)\). 
We can get \(FNR(F) \leq o(1)\).
\end{proof}
Theorem 2 reveals that a calibrated ensemble distillation model can discriminate both multi-view ID data and single-view ID data from OOD. 

A smaller FNR means less ID data will be detected as OOD while correctly detecting all OOD data, \textit{i.e.}, a more discriminative boundary of ID and OOD. 

\section{C. Derivation of Parameter Update}
\label{Derivation of Parameter Update}
For parameter update of MVOL, we have following derivation.
\begin{equation}
\frac{\partial C}{\partial z_i} = \sum_{j = 1}^{k}p_j q_i - p_i
\end{equation}
Where $C$ is the loss term using OOD data.
\begin{proof}[Derivation]
\begin{equation}
C = -\sum_{j = 1}^k p_j\log q_j = - \sum_{j \neq i}^k p_j\log q_j - p_i\log q_i = C_1 + C_2
\end{equation}
\begin{align*}
\frac{\partial C}{\partial z_i} &= \frac{\partial C_1}{\partial z_i} + \frac{\partial C_2}{\partial z_i}\\
&= -\sum_{j \neq i}^{k}\frac{\partial (p_j\log q_j)}{\partial q_j} \frac{\partial q_j}{\partial z_i}  -\frac{\partial(  p_i\log q_i)}{\partial q_i} \frac{\partial q_i}{\partial z_i}\\
&= - \sum_{j \neq i}^{k}\frac{p_j}{q_j} \frac{\partial q_j}{\partial z_i} - \frac{p_i}{q_i} \frac{\partial q_i}{\partial z_i}\\
&= \sum_{j \neq i}^{K}p_j q_i + p_i(q_i - 1)\\
&= \sum_{j = 1}^{K}p_j q_i - p_i
\end{align*}
\end{proof}
For parameter update of OE, we have following derivation.
\begin{equation}
\frac{\partial C}{\partial z_i} = q_i - \frac{1}{k} =\frac{e^{z_i}}{\sum_j e^{z_j}} - \frac{1}{k}
\end{equation}
Where $C$ is the loss term using OOD data.
\begin{proof}[Derivation]
\begin{equation}
C = -\sum_{j = 1}^k \frac{1}{k}\log q_j = - \sum_{j \neq i}^k \frac{1}{k}\log q_j - \frac{1}{k}\log q_i = C_1 + C_2
\end{equation}
\begin{align*}
\frac{\partial C}{\partial z_i} &= \frac{\partial C_1}{\partial z_i} + \frac{\partial C_2}{\partial z_i}\\
&= \frac{k - 1}{k} q_i + \frac{1}{k}(q_i - 1)\\
&= q_i - \frac{1}{k}
\end{align*}
\end{proof}
\section{D. Future Work}
\label{Limitations and Future Work}
In this work, we proposed an extended MVDM to formalize the correlations between ID and OOD data. 
With this attempt, we can further study more problems in OOD detection. (1) The outlier sampling strategy for Outlier Exposure. (2) Seeking efficient ID objective to calibrate logits of models without auxiliary OOD dataset. 

As shown in Figure 1 (c), the model can produces unexpectedly low MaxLogit for ID data, which harms the performance of MaxLogit as the OOD scoring function. Improving OOD detection from this factor could also be a potential research direction.

\section{E. Experimental Details and Additional Results}
\subsection{Training Details}
\label{Training Details}
For all experiments and methods, we use the Wide ResNet \citep{zagoruyko2016wideresnet} architecture with 40 layers and a widen factor of 2. Models are optimized using stochastic gradient descent with Nesterov momentum \citep{sgd}. The batch size is 128 for both ID and OOD training data. The weight decay coefficient is $0.0005$, and the momentum is $0.09$. There are 100 epochs in all experiments. The initial learning rate is 0.1 and decayed by a factor of 10 after 50, 75, and 90 epochs. The weight $\beta$ is chosen as the same as \citep{hendrycks2018oe},\textit{i.e.}, $0.5$. The hyperparameter $\epsilon$ is $0.1$ for CIFAR-10 and $0.02$ for CIFAR-100. We use the ensemble of 10 independently trained models with random initialization as a teacher model when performing ensemble distillation. The distillation algorithm is from the original paper \citep{hinton2015distilling} referred to \citep{allen-zhu2023towards}. The temperature hyperparameter is 2 for both CIFAR-10 and CIFAR-100. The loss weight of cross-entropy loss is 0.5 for CIFAR-100 and 0.1 for CIFAR-10. We generate soft labels for the ID training set together before training models to reduce training time. We run all experiments using PyTorch on NVIDIA GeForce RTX 2080 GPUs. We run 5 independent training runs with random initializations for each experiment. The random initialization in this paper is achieved by setting different random seeds for each independent run in practice, referred to \citep{allen-zhu2023towards}. All baselines are reproduced based on official repositories.
\subsection{Additional Experimental Setups}
\label{Additional Experimental Setups}
Figure 1: (a)
We train ten Wide ResNets \citep{zagoruyko2016wideresnet} with random initialization solely on CIFAR-10. An OOD image from 300K RandomImages is fed into these models. The obtained results are averages of logits of these models. All error values are given as standard deviations.  (b) We test one model in (a) on the CIFAR-10 test set and six OOD test sets. The average MaxLogit on each set is used in (b). (c) We utilize one model in (a) to infer the CIFAR-10 test set and Textures OOD test sets. With the obtained MaxLogit on these sets, we draw a density graph. 

\subsection{Test OOD Datsets}
\label{Test OOD Datsets}
Following \citep{liu2020energy, chen2021atom, ming2022poem}, we use six natural image datasets as OOD test datasets. They are SVHN \citep{SVHN}, Textures \citep{Textures}, Places365 \citep{places365}, LSUN (crop) \citep{LSUN}, LSUN (resize) \citep{LSUN} , and iSUN \citep{iSUN}. All images are $32 * 32$ pixels the same as CIFAR images.

\textbf{SVHN.} The SVHN dataset contains color images of house numbers, categorized into ten classes representing digits 0-9. The SVHN OOD test set contains 1,000 randomly selected test images for each class for evaluation. There are 10,000 images in total.

\textbf{Places365.} The Places365 dataset contains large-scale photographs of scenes categorized into 365 distinct scene categories. Each category in the test set comprises 900 images. For OOD detection evaluation, the Places365 test set includes a randomly sampled subset of 10,000 images.

\textbf{Textures.} The Describable Textures Dataset contains textural images captured in natural settings. The entire collection, comprising 5,640 images, is used for evaluation.

\textbf{LSUN (crop) and LSUN (resize).}
The Large-scale Scene Understanding dataset (LSUN) includes a test set containing 10,000 images from 10 scenes. Two subsets, LSUN-C and LSUN-R, are created by randomly cropping images to 32×32 pixels and downsampling images to 32×32 pixels, respectively.

\textbf{iSUN.}
The iSUN dataset consists of a subset of SUN images, comprising the entire collection of 8,925 images.

\subsection{More Visualizations}
\label{More Visualizations}
Figure \ref{IDAttributes_Appendix} presents more examples illustrating that OOD data could contain ID attributes. In all these OOD images, pre-trained neural networks show higher logits on specific categories than the others.

Figure \ref{optimization_effect_Appendix} presents more examples to show the different optimization objectives of OE and our MVOL on auxiliary OOD data. In these cases, we can observe that OE (red) tends to produce logits with a much smaller magnitude than MVOL (blue). Negative logits with smaller magnitudes suggest OE overemphasizes noise on type I categories. Positive logits with smaller magnitudes suggest OE underemphasizes minor ID features on type II categories. 
\begin{figure*}[htbp]
\centering
\includegraphics[scale=0.5]{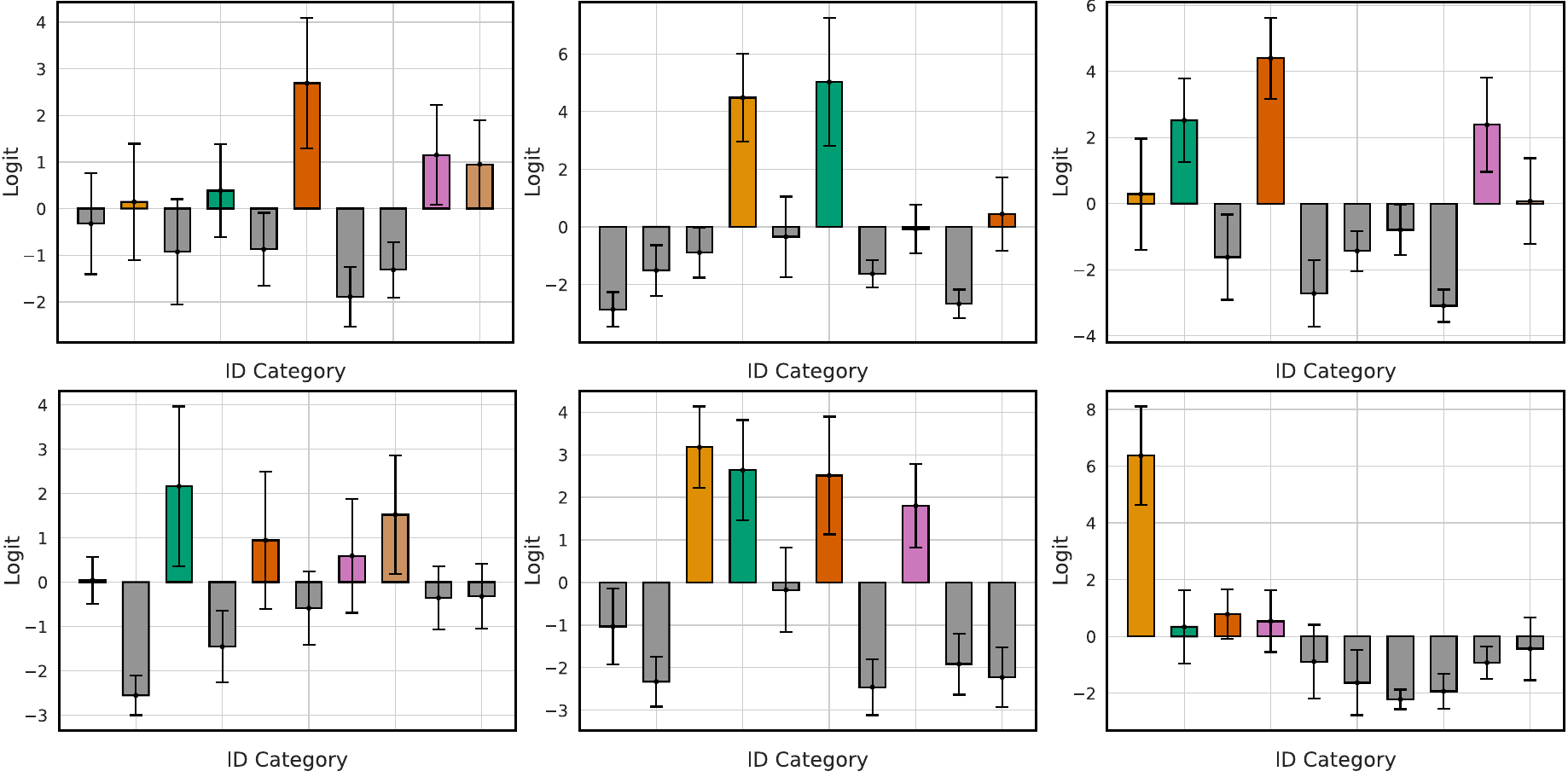}
\caption{More examples for illustrating that OOD data could have ID attributes.}
\label{IDAttributes_Appendix}
\end{figure*}
\begin{figure*}[htbp]
\centering
\includegraphics[scale=0.5]{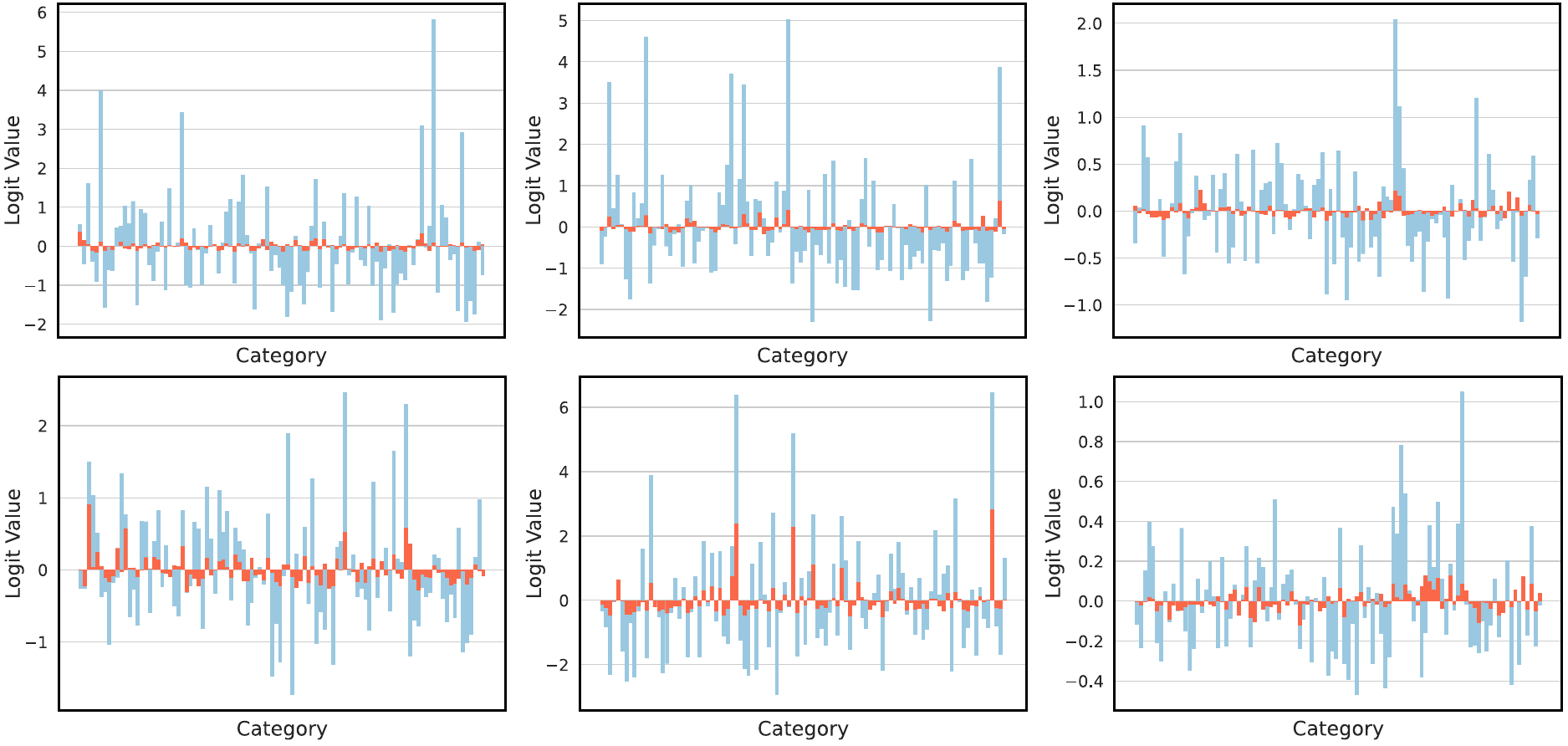}
\caption{More examples for illustrating the optimization objectives of OE and Our MVOL on auxiliary OOD samples. Blue bars represent our MVOL, and red bars represent OE. CIFAR-100 and 300K RandomImages provide ID and auxiliary OOD data, respectively.}
\label{optimization_effect_Appendix}
\end{figure*}
\subsection{Sensitivity Analysis}
\label{Sensitivity Analysis}
To study how hyperparameter $\epsilon$ affects MVOL, we create a validation set by randomly sampling $10,000$ images from 300K RandomImages. These validation images are disjoint with our test datasets. We conduct experiments on both CIFAR-10 and CIFAR-100. The results are shown in Table \ref{sensitivity}. A mild value of $\beta$ generally works well.  A large $\epsilon$ indicates fewer logits will be considered as type II categories and the loss of more information to moderate the logits, leading to a worse calibration effect. A small $\epsilon$ indicates that more logits will be considered type II categories. It introduces more noise when performing calibration, leading to a worse calibration effect. We choose 0.100 for CIFAR-10 and 0.02 for CIFAR-100.
\begin{table}[htbp]
  \centering
  \tabcolsep=0.14cm
  \caption{Sensitivity Analysis on $\epsilon$}
    \begin{tabular}{cccccccc}
    \toprule
     \multirow{2}[0]{*}{$\epsilon$}&\multicolumn{3}{c}{CIFAR-10}  & \multirow{2}[0]{*}{$\epsilon$}&\multicolumn{3}{c}{CIFAR-100} \\
    \cmidrule(lr){2-4} \cmidrule(lr){6-8}
     & FPR95 $\downarrow$ & AUROC $\uparrow$ & Acc $\uparrow$   &  & FPR95 $\downarrow$ & AUROC $\uparrow$ & Acc $\uparrow$ \\
    \midrule
    0.050  & 6.07  & 97.56 & \textbf{94.82} & 0.005 & 36.74 & 91.84 & \textbf{74.74} \\
    0.075 & 6.05  & 97.70  & 94.60  & 0.010  & 34.00    & 92.20  & 74.22 \\
    0.100   & \textbf{5.96}  & \textbf{97.71} & 94.65 & 0.015 & \textbf{33.70}  & 92.40  & 74.29 \\
    0.125 & 6.34  & 97.62 & 94.59 & 0.020  & 33.90  & \textbf{92.50}  & 74.69 \\
    0.150  & 6.52  & 97.60  & 94.57 & 0.025 & 33.99 & 92.35 & 74.29 \\
    0.200   & 6.57  & 97.52 & 94.65 & 0.030  & 35.54 & 92.25 & 74.26 \\
    0.300   & 7.98  & 97.51 & 94.54 & 0.040  & 34.99 & 92.46 & 74.23 \\
    \bottomrule
    \end{tabular}%
  \label{sensitivity}%
\end{table}%

\begin{table}[htbp]
  \centering
  \caption{Additional results provided to understand MVOL. Experiments are conducted in the single model setting. ID and auxiliary OOD datasets are CIFAR-100 and 300K RandomImages, respectively.}
    \begin{tabular}{llll}
    \toprule
    Method &FPR95 $\downarrow$ & AUROC $\uparrow$ & Acc $\uparrow$ \\
    \midrule
    ZeroMVOL  & \multicolumn{1}{l}{54.06{\tiny \ $\pm$ 3.49}} & \multicolumn{1}{l}{87.28{\tiny \ $\pm$ 1.34}} & \multicolumn{1}{l}{\textbf{74.82}{\tiny \ $\pm$ 0.38}} \\
    MVOL  & \textbf{42.96}{\tiny \ $\pm$ 0.86} & \textbf{90.69}{\tiny \ $\pm$ 0.26} & 74.29{\tiny \ $\pm$ 0.33} \\
    \bottomrule
    \end{tabular}%
  \label{closer}%
\end{table}%
\begin{table*}[htbp]
  \caption{Main results on MisD. The best result is in bold, and the second-best result is underlined.}
  \centering
  {
    \tabcolsep=0.15cm
    \begin{tabular}{llllllll}
    \toprule
     \multirow{2}[0]{*}{Method} & \multicolumn{3}{c}{CIFAR-10} & \multicolumn{3}{c}{CIFAR-100} \\
     \cmidrule(lr){2-4} \cmidrule(lr){5-7}
     & FPR95 $\downarrow$ & AUROC $\uparrow$ & Acc $\uparrow$& FPR95 $\downarrow$& AUROC $\uparrow$& Acc $\uparrow$\\
    \midrule
    &       &       \multicolumn{3}{c}{Single Model Setting} & \\
    \midrule
    MSP   & \textbf{38.68}{\tiny \ $\pm$ 2.01} & \underline{92.30}{\tiny \ $\pm$ 0.29} & 94.22{\tiny \ $\pm$ 0.13} & \textbf{64.73}{\tiny \ $\pm$ 0.95} & \textbf{85.99}{\tiny \ $\pm$ 0.32} & \textbf{74.81}{\tiny \ $\pm$ 0.30} \\
    OE    & 45.72{\tiny \ $\pm$ 1.55} & 91.25{\tiny \ $\pm$ 0.29} & \underline{94.46}{\tiny \ $\pm$ 0.19} & 71.21{\tiny \ $\pm$ 1.18} & 84.13{\tiny \ $\pm$ 0.41} & 73.86{\tiny \ $\pm$ 0.23} \\
    MVOL (Ours)  & \underline{41.67}{\tiny \ $\pm$ 3.44} & \textbf{93.05}{\tiny \ $\pm$ 0.14} & \textbf{94.68}{\tiny \ $\pm$ 0.11} & \underline{69.62}{\tiny \ $\pm$ 1.57} & \underline{84.98}{\tiny \ $\pm$ 0.47} & \underline{74.38}{\tiny \ $\pm$ 0.35} \\

    \midrule
    &       &    \multicolumn{3}{c}{ensemble distillation model setting} &  \\
    \midrule
    MSP   & \textbf{38.09}{\tiny \ $\pm$ 2.10} & \underline{92.77}{\tiny \ $\pm$ 0.39} & 94.27{\tiny \ $\pm$ 0.13} & \textbf{64.05}{\tiny \ $\pm$ 3.25} & \underline{86.47}{\tiny \ $\pm$ 0.53} & \textbf{76.99}{\tiny \ $\pm$ 0.16} \\
    OE    & 42.20{\tiny \ $\pm$ 1.51} & 92.47{\tiny \ $\pm$ 0.28} & \textbf{94.72}{\tiny \ $\pm$ 0.23} & 71.25{\tiny \ $\pm$ 1.66} & 85.99{\tiny \ $\pm$ 0.31} & 74.94{\tiny \ $\pm$ 0.19} \\
    MVOL (Ours)  & \underline{39.94}{\tiny \ $\pm$ 2.27} & \textbf{93.04}{\tiny \ $\pm$ 0.34} & \underline{94.69}{\tiny \ $\pm$ 0.17} & \underline{64.61}{\tiny \ $\pm$ 1.00} & \textbf{86.70}{\tiny \ $\pm$ 0.36} & \underline{76.43}{\tiny \ $\pm$ 0.44} \\
    \bottomrule
    \end{tabular}%
    }
  \label{faliure detection}%
\end{table*}%
\subsection{A Closer Look at Multi-view based Learning Objective in MVOL}
\label{A Closer Look at MVOL}
When MVOL calibrates logits on auxiliary outliers, logits of category $i\in S_{j, I}^{(t)}$ increase according to their softmax-normalized values. 

In other words, a lower $p_{j, i}$ in Equation (6) leads to a smaller gradient weight and usually increases fewer logits. In this section, we conduct experiments to illustrate the logic of this operation. We set $p_{j,i} = 0$ $i\in S_{j,I}^{(t)}$, denoted as ZeroMVOL, a control group. We use CIFAR-100 as an ID dataset. The main results are shown in Table \ref{closer}. We can observe that ZeroMVOL performs much worse than MVOL. Specifically, ZeroMVOL increases FPR95 by $11.1\%$ compared with MVOL. This means simply suppressing logits of $i\in S_{j, I}^{(t)}$ can harm logit calibration. The reason could be that it is hard to guarantee that $S_{j, I}^{(t)}$ are all type I categories via a constant threshold $\epsilon$. $S_{j, I}^{(t)}$ could also contain type II categories with minor ID features associated with them. With our MVOL, these categories could have larger softmax-normalized values and gain a more logit-increasing gradient. Meanwhile, this operation also maintains less influence on these categories associated with only noise due to their smaller softmax-normalized values. This way could avoid overemphasizing noise and mitigate the limitation of threshold-based methods, \textit{i.e}., the performance is sensitive to the threshold value.
\subsection{Mis-classification Detection Performance}
\label{Mis-classification Detection Performance}
As demonstrated in \citep{zhu2023openmix}, employing auxiliary outliers for OOD detection may decrease models' misclassification detection (MisD) performance. We assess MVOL's impact on alleviating this degradation. A model trained solely on an ID data serves as a strong baseline  \citep{hendrycks2017MSP}. The maximum softmax probability (MSP) is used as a scoring function, a common practice in this area. We use common metrics in the literature to evaluate MisD performance. They are the area under the receiver operating characteristic curve (AUROC) and the false positive rate with a $95\%$ true positive rate (FPR95). A false positive is a misclassified data that is identified as correct. We also report the classification accuracy (Acc) as a reference. The results are depicted in Table \ref{faliure detection}. MVOL outperforms OE in both single model and ensemble distillation model settings. For example, in the ensemble distillation model setting on CIFAR-100, MVOL achieves $64.61\%$ on FPR95, a slight increase of $0.56\%$ over MSP. Conversely, OE exhibits a much more increase of $7.2\%$ than MVOL. This performance gain in MisD could be attributed to MVOL's unified perspective of OOD and ID. MVOL avoids over-suppressing logit responses to ID features, which help to maintain a high confidence for correctly classified samples.

Additionally, we evaluate the MisD performance of models trained with wild datasets. The results are presented in Table \ref{noise misd}. In both settings, MVOL consistently maintains strong MisD performance. For example, as the noise level increases from 0.05 to 0.5 in the ensemble distillation model setting, MVOL sustains the AUROC metric above $91\%$.  In contrast, OE experiences a reduction in AUROC from $91.39\%$ to $84.84\%$, marking a substantial decline. This MisD performance gain could be attributed to MVOL's rationally handling ID attributes in auxiliary outliers, rather than blindly suppressing them. 

\begin{table*}[htbp]
  \centering
  \caption{Main results on MisD performance of models trained with wild datasets.}
    {
    \tabcolsep=0.15cm
    \begin{tabular}{llllllll}
    \toprule
    \multirow{2}[0]{*}{Noise} &\multirow{2}[0]{*}{Method}& \multicolumn{3}{c}{Single Model Setting} & \multicolumn{3}{c}{ensemble distillation model setting} \\
          \cmidrule(lr){3-5}\cmidrule(lr){6-8}
           & & FPR95 $\downarrow$ & AUROC$\uparrow$& Acc$ \uparrow$   & FPR95 $\downarrow$ & AUROC $\uparrow$& Acc $\uparrow$ \\
    \midrule
    $\backslash$& MSP   & \textcolor[RGB]{128,128,128}{49.17{\tiny \ $\pm$ 1.29}} & \textcolor[RGB]{128,128,128}{90.92{\tiny \ $\pm$ 0.23}} & \textcolor[RGB]{128,128,128}{91.38{\tiny \ $\pm$ 0.27}} & \textcolor[RGB]{128,128,128}{48.31{\tiny \ $\pm$ 2.00}} & \textcolor[RGB]{128,128,128}{90.68{\tiny \ $\pm$ 0.37}} & \textcolor[RGB]{128,128,128}{91.94{\tiny \ $\pm$ 0.09}} \\ 
    \midrule
    \multirow{2}[0]{*}{$\alpha = 0$} &     OE    & \textbf{52.18}{\tiny \ $\pm$ 1.68} & 91.20{\tiny \ $\pm$ 0.11} & \textbf{92.24}{\tiny \ $\pm$ 0.09} & 51.85{\tiny \ $\pm$ 2.29} & \textbf{92.28}{\tiny \ $\pm$ 0.30} & \textbf{92.32}{\tiny \ $\pm$ 0.16} \\
    & MVOL (Ours)  & 53.84{\tiny \ $\pm$ 2.17} & \textbf{91.28}{\tiny \ $\pm$ 0.18} & 91.71{\tiny \ $\pm$ 0.10} & \textbf{46.46}{\tiny \ $\pm$ 1.04} & 91.96{\tiny \ $\pm$ 0.32} & 91.96{\tiny \ $\pm$ 0.14} \\       
    \midrule
    \multirow{2}[0]{*}{$\alpha = 0.05$} &     OE    & 54.41{\tiny \ $\pm$ 1.71} & 89.41{\tiny \ $\pm$ 0.23} & 91.29{\tiny \ $\pm$ 0.22} & 54.37{\tiny \ $\pm$ 1.77} & 91.39{\tiny \ $\pm$ 0.40} & 91.51{\tiny \ $\pm$ 0.20} \\
    & MVOL (Ours)  & \textbf{53.84}{\tiny \ $\pm$ 2.46} & \textbf{91.35}{\tiny \ $\pm$ 0.35} & \textbf{91.75}{\tiny \ $\pm$ 0.25} & \textbf{47.68}{\tiny \ $\pm$ 2.23} & \textbf{91.69}{\tiny \ $\pm$ 0.37} & \textbf{91.74}{\tiny \ $\pm$ 0.34} \\       
    \midrule
    \multirow{2}[0]{*}{$\alpha = 0.1$} &     OE    & 55.92{\tiny \ $\pm$ 1.45} & 88.20{\tiny \ $\pm$ 0.73} & 90.92{\tiny \ $\pm$ 0.35} & 55.55{\tiny \ $\pm$ 2.12} & 90.67{\tiny \ $\pm$ 0.26} & 91.01{\tiny \ $\pm$ 0.24} \\
    & MVOL (Ours)  & \textbf{53.80}{\tiny \ $\pm$ 2.66} & \textbf{91.36}{\tiny \ $\pm$ 0.41} & \textbf{91.55}{\tiny \ $\pm$ 0.34} & \textbf{46.92}{\tiny \ $\pm$ 2.99} & \textbf{91.52}{\tiny \ $\pm$ 0.60} & \textbf{91.64}{\tiny \ $\pm$ 0.31} \\       
    \midrule
    \multirow{2}[0]{*}{$\alpha = 0.3$} &     OE    & 59.75{\tiny \ $\pm$ 1.93} & 85.88{\tiny \ $\pm$ 0.50} & 89.04{\tiny \ $\pm$ 0.33} & 62.00{\tiny \ $\pm$ 1.95} & 87.39{\tiny \ $\pm$ 0.30} & 88.60{\tiny \ $\pm$ 0.15} \\
    &MVOL (Ours)  & \textbf{53.50}{\tiny \ $\pm$ 1.47} & \textbf{90.93}{\tiny \ $\pm$ 0.24} & \textbf{91.24}{\tiny \ $\pm$ 0.17} & \textbf{46.60}{\tiny \ $\pm$ 1.79} & \textbf{91.69}{\tiny \ $\pm$ 0.54} & \textbf{91.76}{\tiny \ $\pm$ 0.08} \\        
    \midrule
    \multirow{2}[0]{*}{$\alpha = 0.5$} &     OE    & 61.47{\tiny \ $\pm$ 2.37} & 84.99{\tiny \ $\pm$ 0.99} & 88.04{\tiny \ $\pm$ 0.20} & 66.68{\tiny \ $\pm$ 1.34} & 84.84{\tiny \ $\pm$ 0.45} & 86.80{\tiny \ $\pm$ 0.23} \\
    & MVOL (Ours)  & \textbf{52.35}{\tiny \ $\pm$ 1.46} & \textbf{91.05}{\tiny \ $\pm$ 0.45} & \textbf{90.79}{\tiny \ $\pm$ 0.24} & \textbf{47.30}{\tiny \ $\pm$ 2.03} & \textbf{91.68}{\tiny \ $\pm$ 0.47} & \textbf{91.81}{\tiny \ $\pm$ 0.22} \\       
    \bottomrule
    \end{tabular}%
    }
  \label{noise misd}%
\end{table*}%
\subsection{OOD Detection with Fine-tuning}
\label{OOD Detection with Fine-tuning}

In this section, we evaluate MVOL using the fine-tuning technique. When conducting these experiments, we independently train each model five runs with random initialization. Specifically, there are 50 epochs without an auxiliary OOD dataset followed by another 50 epochs with that. 
All other hyper-parameters remain consistent with those used in Table 1. The main results are presented in Table \ref{fine-tune}. Through comparisons, we make two main observations. (1) MVOL consistently outperforms other fine-tune-based methods regarding ID accuracy. (2) MVOL overall excels in OOD detection across both settings, achieving the best or second-best results. Additionally, we observe that MVOL's fine-tuning performance is not as strong as MVOL's from-scratch training in Table 1. The reason could be that with fine-tuning, methods such as confidence loss in OE + MaxLogit and Energy Bounded Penalty in Energy w/Aux tend to suppress minor ID features and emphasize noise in auxiliary outliers to a lesser extent.

\begin{table*}[htbp]
    \centering
    \caption{Main results on OOD detection with fine-tuning. The best result is in bold, and the second-best result is underlined.}
    \tabcolsep=0.15cm
    \begin{tabular}{llllllll}
    \toprule
    \multirow{2}[0]{*}{Method}& \multicolumn{3}{c}{CIFAR-10} & \multicolumn{3}{c}{CIFAR100} \\
    \cmidrule(lr){2-4}\cmidrule(lr){5-7}
     & FPR95 $\downarrow$ & AUROC $\uparrow$& Acc $\uparrow$   & FPR95 $\downarrow$ & AUROC $\uparrow$& Acc $\uparrow$ \\
        \midrule
        &       &       \multicolumn{3}{c}{Single Model Setting} & \\
        \midrule
         Energy w/Aux & \underline{4.08}{\tiny \ $\pm$ 0.33} & 98.03{\tiny \ $\pm$ 0.21} & 92.33{\tiny \ $\pm$ 0.21} & 44.89{\tiny \ $\pm$ 3.95} & \textbf{91.00}{\tiny \ $\pm$ 0.94} & 70.25{\tiny \ $\pm$ 0.32} \\
         OE & 4.88{\tiny \ $\pm$ 0.18} & 98.36{\tiny \ $\pm$ 0.06} & \underline{94.17}{\tiny \ $\pm$ 0.06} & 45.99{\tiny \ $\pm$ 0.69} & 89.80{\tiny \ $\pm$ 0.34} & 73.58{\tiny \ $\pm$ 0.22} \\
         OE + MaxLogit & 4.73{\tiny \ $\pm$ 0.15} & \underline{98.39}{\tiny \ $\pm$ 0.07} & 94.17{\tiny \ $\pm$ 0.06} & 45.59{\tiny \ $\pm$ 0.83} & 90.56{\tiny \ $\pm$ 0.36} & 73.58{\tiny \ $\pm$ 0.22} \\
        MVOL(ours) & \textbf{3.88}{\tiny \ $\pm$ 0.24} & \textbf{98.63}{\tiny \ $\pm$ 0.05} & \textbf{94.45}{\tiny \ $\pm$ 0.10} & \textbf{43.15}{\tiny \ $\pm$ 2.29} & \underline{90.81}{\tiny \ $\pm$ 0.75} & \textbf{74.28}{\tiny \ $\pm$ 0.27} \\
        \midrule
        &       &       \multicolumn{3}{c}{Ensemble Distillation Model Setting} & \\
        \midrule
        Energy w/Aux & \textbf{3.66}{\tiny \ $\pm$ 0.16} & 98.12{\tiny \ $\pm$ 0.07} & 92.95{\tiny \ $\pm$ 0.10} & 45.54{\tiny \ $\pm$ 2.80} & \underline{90.75}{\tiny \ $\pm$ 0.57} & 72.81{\tiny \ $\pm$ 0.30} \\
        OE & 4.91{\tiny \ $\pm$ 0.27} & 98.47{\tiny \ $\pm$ 0.07} & \underline{94.25}{\tiny \ $\pm$ 0.14} & 42.83{\tiny \ $\pm$ 1.38} & 89.93{\tiny \ $\pm$ 0.67} & \underline{75.13}{\tiny \ $\pm$ 0.30} \\
        OE + MaxLogit & 4.76{\tiny \ $\pm$ 0.28} & \underline{98.52}{\tiny \ $\pm$ 0.07} & 94.25{\tiny \ $\pm$ 0.14} & \underline{42.67}{\tiny \ $\pm$ 1.62} & \textbf{91.21}{\tiny \ $\pm$ 0.54} & 75.13{\tiny \ $\pm$ 0.30} \\
        MVOL (ours) & \underline{3.96}{\tiny \ $\pm$ 0.24} & \textbf{98.53}{\tiny \ $\pm$ 0.07} & \textbf{94.48}{\tiny \ $\pm$ 0.16} & \textbf{42.61}{\tiny \ $\pm$ 2.92} & 88.75{\tiny \ $\pm$ 0.72} & \textbf{76.48}{\tiny \ $\pm$ 0.28} \\
        \midrule
    \end{tabular}
    \label{fine-tune}
\end{table*}
\subsection{Self-distillation as an Implicit Ensemble}
\label{Self-distillation as an Implicit Ensemble}

\citep{allen-zhu2023towards} suggest that self-distillation can function as an implicit ensemble distillation, where soft labels are generated by a single model rather than ensemble models. This approach enables a single model to learn more features than the single model setting. In this section, we evaluate MVOL across ensemble distillation and self-distillation settings. The main results are presented in Table \ref{selfdistill}. 
We observe that MVOL in ensemble knowledge distillation typically outperforms that in self-knowledge distillation setting. 
The reason could be that models with self-distillation can not learn as many as ensemble distillation, which has proved in  \cite{allen-zhu2023towards}. 
For simplicity, in the main submission, we use ensemble distillation to analyze the case where OOD detectors learn more features.

\begin{table*}[htbp]
  \centering
  \caption{Main results on using self-distillation as an implicit ensemble. SelfKD means self-knowledge distillation and EKD means ensemble knowledge distillation.}
    \begin{tabular}{lllllll}
    \toprule
    \multicolumn{7}{c}{CIFAR-10} \\
    \midrule
    \multirow{2}[0]{*}{Method} & \multicolumn{2}{c}{LSUN} & \multicolumn{2}{c}{LSUN (resize)} & \multicolumn{2}{c}{iSUN} \\
    \cmidrule(lr){2-3}     \cmidrule(lr){4-5}    \cmidrule(lr){6-7}
          & FPR $\downarrow$   & AUROC $\uparrow$ & FPR $\downarrow$   & AUROC $\uparrow$ & FPR $\downarrow$   & AUROC $\uparrow$ \\
    \midrule
    MVOL w/ SelfKD & 0.77{\tiny \ $\pm$ 0.17} & 99.31{\tiny \ $\pm$ 0.08} & 1.82{\tiny \ $\pm$ 0.41} & 99.03{\tiny \ $\pm$ 0.11} & 1.96{\tiny \ $\pm$ 0.44} & 99.05{\tiny \ $\pm$ 0.13} \\
    MVOL w/ EKD & \textbf{0.76}{\tiny \ $\pm$ 0.1} & \textbf{99.45}{\tiny \ $\pm$ 0.15} & \textbf{1.37}{\tiny \ $\pm$ 0.38} & \textbf{99.13}{\tiny \ $\pm$ 0.1} & \textbf{1.70}{\tiny \ $\pm$ 0.5} & \textbf{99.07}{\tiny \ $\pm$ 0.12} \\
    \midrule
    \multirow{2}[0]{*}{Method} & \multicolumn{2}{c}{Textures} & \multicolumn{2}{c}{places365} & \multicolumn{2}{c}{SVHN} \\
    \cmidrule(lr){2-3}     \cmidrule(lr){4-5}    \cmidrule(lr){6-7}
          & FPR $\downarrow$   & AUROC $\uparrow$ & FPR $\downarrow$   & AUROC $\uparrow$ & FPR $\downarrow$  & AUROC $\uparrow$ \\
    \midrule
    MVOL w/ SelfKD & 4.59{\tiny \ $\pm$ 0.37} & 98.10{\tiny \ $\pm$ 0.19} & 10.99{\tiny \ $\pm$ 0.46} & 96.38{\tiny \ $\pm$ 0.15} & \textbf{1.48}{\tiny \ $\pm$ 0.27} & 99.21{\tiny \ $\pm$ 0.07} \\
    MVOL w/ EKD & \textbf{4.34}{\tiny \ $\pm$ 0.21} & \textbf{98.22}{\tiny \ $\pm$ 0.09} & \textbf{10.78}{\tiny \ $\pm$ 0.39} & \textbf{96.43}{\tiny \ $\pm$ 0.18} & 1.57{\tiny \ $\pm$ 0.35} & \textbf{99.24}{\tiny \ $\pm$ 0.14} \\
    \midrule
    \multicolumn{7}{c}{CIFAR-100} \\
    \midrule
    \multirow{2}[0]{*}{Method} & \multicolumn{2}{c}{LSUN} & \multicolumn{2}{c}{LSUN (resize)} & \multicolumn{2}{c}{iSUN} \\
    \cmidrule(lr){2-3}     \cmidrule(lr){4-5}    \cmidrule(lr){6-7}
          & FPR $\downarrow$   & AUROC $\uparrow$ & FPR $\downarrow$  & AUROC $\uparrow$ & FPR $\downarrow$  & AUROC $\uparrow$ \\
    \midrule
    MVOL w/ SelfKD & 15.34{\tiny \ $\pm$ 1.92} & 96.58{\tiny \ $\pm$ 0.41} & 45.23{\tiny \ $\pm$ 10.1} & 85.68{\tiny \ $\pm$ 4.07} & 51.79{\tiny \ $\pm$ 9.95} & 82.82{\tiny \ $\pm$ 4.2} \\
    MVOL w/ EKD & \textbf{14.16}{\tiny \ $\pm$ 1.66} & \textbf{96.84}{\tiny \ $\pm$ 0.37} & \textbf{33.09}{\tiny \ $\pm$ 5.12} & \textbf{91.45}{\tiny \ $\pm$ 1.71} & \textbf{40.50}{\tiny \ $\pm$ 5.0} & \textbf{89.16}{\tiny \ $\pm$ 1.65} \\
    \midrule
    \multirow{2}[0]{*}{Method} & \multicolumn{2}{c}{Textures} & \multicolumn{2}{c}{places365} & \multicolumn{2}{c}{SVHN} \\
    \cmidrule(lr){2-3}     \cmidrule(lr){4-5}    \cmidrule(lr){6-7}
          & FPR $\downarrow$  & AUROC $\uparrow$& FPR $\downarrow$  & AUROC $\uparrow$ & FPR $\downarrow$  & AUROC $\uparrow$\\
    \midrule
    MVOL w/ SelfKD & 50.54{\tiny \ $\pm$ 2.1} & 85.40{\tiny \ $\pm$ 0.74} & 54.34{\tiny \ $\pm$ 0.59} & 83.21{\tiny \ $\pm$ 0.13} & \textbf{29.25}{\tiny \ $\pm$ 9.3} & \textbf{94.24}{\tiny \ $\pm$ 1.55} \\
    MVOL w/ EKD & \textbf{48.31}{\tiny \ $\pm$ 2.28} & \textbf{86.54}{\tiny \ $\pm$ 0.8} & \textbf{54.24}{\tiny \ $\pm$ 0.34} & \textbf{83.53}{\tiny \ $\pm$ 0.2} & 33.40{\tiny \ $\pm$ 3.09} & 93.10{\tiny \ $\pm$ 0.47} \\
    \bottomrule
    \end{tabular}%
  \label{selfdistill}%
\end{table*}%

\end{document}